\documentclass{IEEEtran}
\usepackage{algorithmic}
\usepackage{graphicx}
\usepackage{graphics}
\usepackage{textcomp}
\usepackage{subfigure}
\usepackage{tabularx}
\usepackage{times}
\usepackage{multirow}
\usepackage{bm}
\usepackage{latexsym}
\usepackage{booktabs}
\usepackage{threeparttable}
\usepackage{epsf}
\usepackage[vlined,ruled,linesnumbered]{algorithm2e}
\usepackage{color,colortbl}

\usepackage{cite}
\usepackage{picinpar}
\usepackage{amsmath}
\usepackage{url}    
\usepackage[latin1]{inputenc}
\usepackage{soul}
\usepackage{multirow}
\usepackage{pifont}
\usepackage{alltt}
\usepackage{hyperref}
\usepackage{enumerate}
\usepackage{siunitx}
\usepackage{breakurl}
\usepackage{epstopdf}
\usepackage{pbox}
\usepackage{amsmath,amssymb,amsfonts}
\usepackage{amsthm}
\usepackage{mathrsfs}
\usepackage{makecell}
\usepackage{pbox}
\usepackage{supertabular}
\usepackage{multicol}
\usepackage{lipsum}
\usepackage{bm}
\usepackage{leftindex}
\newtheorem{theorem}{Theorem}

\newtheorem{remark}{Remark}
\newtheorem{lemma}{Lemma}
\newtheorem{definition}{Definition}
\newtheorem{assumption}{Assumption}
\newcommand{\blue}{\textcolor{blue}}

\def\BibTeX{{\rm B\kern-.05em{\sc i\kern-.025em b}\kern-.08em
		T\kern-.1667em\lower.7ex\hbox{E}\kern-.125emX}}
\begin{document}
	\title{Integrated Cooperative Localization of Heterogeneous Measurement Swarm: A Unified Adaptive Method}
	\author{Kunrui Ze, Wei Wang, \emph{IEEE Senior Member}, Guibin Sun, Jiaqi Yan, Kexin Liu,  {Jinhu L\"u}, \emph{IEEE Fellow}
		\thanks{
			This work was supported in part by the National Key Research and Development Program of China under Grant 2022YFB3305600, and in part by the National Natural Science Foundation of China under Grants 625B2017, 62373019, 62503028 and 62141604. 
			\emph{(Corresponding author: Jinhu L\"u.)}}
		\thanks{
			Kunrui Ze, Guibin Sun, Jiaqi Yan and Kexin Liu are with the School of Automation Science and Electrical Engineering, Beihang University, Beijing 100191, China (e-mail: kr\_ze@buaa.edu.cn; sunguibinx@buaa.edu.cn; jqyan@buaa.edu.cn, kxliu@buaa.edu.cn).
			Wei Wang and {Jinhu L\"u} are with the School of Automation Science and Electrical Engineering, Beihang University, Beijing 100191, China, and also with the Zhongguancun Laboratory, Beijing 100083, China (e-mail: w.wang@buaa.edu.cn; jhlu@iss.ac.cn).}}
	
	\maketitle
	
	\begin{abstract}
		The cooperative localization (CL) problem in a heterogeneous measurement swarm with different measurement capabilities is investigated in this work.
		In practice, heterogeneous sensors lead to directed and sparse measurement topologies, whereas most existing CL approaches rely on multilateral localization with restrictive multi-neighbor geometric requirements.
		To overcome this limitation, we enable pairwise relative localization (RL) between neighboring robots using only mutual measurement and odometry information.		
		\blue{A unified adaptive RL estimator is first developed to handle heterogeneous and unidirectional measurements.
			A distributed CL strategy is then developed in which each robot estimates its pose relative to the anchor by composing each pairwise RL estimate with the corresponding neighbor's pose estimate relative to the anchor.
			This guarantees CL under a weakly connected directed measurement topology, which is the least restrictive topological condition among existing results.}
		The proposed method is independent of specific control tasks and is validated through a formation control application.
	\end{abstract}
	
	
	\section{Introduction}
	\blue{Distributed cooperative localization (CL) is a fundamental capability for multi-robot cooperation in unstructured environments where external localization systems such as GPS are unavailable. 
		In distributed CL, each robot estimates its pose relative to an anchor or reference robot using local measurements, with both relative sensing and information exchange limited to its neighbor robots~\cite{ze2025tase,fang2025Auto}.}
	In practical robotic systems, multiple sensing modalities are often deployed to enhance robustness.
	However, due to the possibility of environmental interference or physical damage to sensors, such as the failure of bearing sensors in strong light and dark environments, the sensors available for each robot may be different.
	Such sensing heterogeneity can result in a sparse and directed measurement topology.
	Therefore, it is essential to achieve CL under heterogeneous measurement conditions.
	
	\blue{Existing CL methods for robot swarms with heterogeneous sensing capabilities rely on multilateral geometric constraints.
		A single range or bearing measurement to one neighbor generally leaves the robot's pose ambiguous, even if that neighbor's pose relative to the anchor is known.
		To resolve this ambiguity, each robot therefore combines simultaneous measurements to \textbf{multiple neighbors} with their known or estimated relative poses to the anchor to determine its own pose relative to the anchor.
		This requires enough neighbors in a nondegenerate geometric configuration.
		It therefore places requirements on both the structure of the measurement graph and the relative placement of the robots. }
	For example, existing methods adopt a multilayer heterogeneous sensing architecture~\cite{xie2023tac}, use Henneberg constructions or twin leader-follower frameworks~\cite{lv2024tii,van2020Auto}, or explicitly prescribe the number and relative placement of neighbors~\cite{boughellaba2023lcs,zheng2025Auto}. 
	Other studies characterize which sets of inter-robot measurements ensure network localizability~\cite{xie2025tro,zhao2016Auto,ze2025tase} or require the measurement graph to contain a directed spanning tree~\cite{lcs2024bounded,tnse2025}.
	
	While these multilateral localization mechanisms provide sufficient constraints for pose estimation, they significantly limit applicability in heterogeneous sensing environments. 
	In practice, relative measurements may be sparse or even unavailable for certain robots due to sensor diversity or failures, rendering such topological conditions difficult to satisfy. 
	\blue{This raises the question of how distributed CL can be achieved under weaker assumptions on the measurement topology when individual robots may lack simultaneous measurements to multiple neighbors.}
	
	\blue{Motivated by these limitations, we propose a unified adaptive CL method for robot swarms with heterogeneous sensing capabilities.
		For any pair of neighboring robots, readily available onboard odometry is used to formulate an observation equation that relates the pair's relative measurements to its unknown relative pose parameters.
		The geometric ambiguity in the pair's relative pose is resolved by combining measurements recorded at different times, thereby enabling relative localization (RL).
		The resulting pairwise RL estimates are then incorporated into a distributed CL estimator.
		Each robot can estimate its own pose relative to the anchor by composing its estimated pose relative to \textbf{a single neighbor} with that neighbor's pose estimate relative to the anchor.
		With bidirectional communication between neighboring robots, information from the anchor can reach every robot distributedly whenever the directed measurement graph is weakly connected. 
		The proposed CL method therefore requires that the measurement graph be weakly connected, one of the least restrictive topological assumptions in the CL literature.
		Finally, the proposed method is validated in a three-dimensional formation control task.}
	
	\textbf{Notations:} 
	Define $\mathcal{O}_{i}$ as the odometer reference frame of robot $i$, which is fixed at the initial time $t_0$, and $\Sigma_{i}$ as the body frame attached to robot $i$.
	At time $t_0$, the two frames coincide in both position and orientation. 
	In this work, the relative orientation between coordinate frames only involves a yaw rotation.
	The rotation matrix between frames 
	$\mathcal{A}$ and $\mathcal{B}$ is parameterized by the relative yaw angle 
	$\theta^{\mathcal{A}}_{\mathcal{B}}$ and defined as
	$\bm{R}^{\mathcal{A}}_{\mathcal{B}}
	=
	\begin{bmatrix}
		\cos \theta^{\mathcal{A}}_{\mathcal{B}} & -\sin \theta^{\mathcal{A}}_{\mathcal{B}} & 0\\
		\sin \theta^{\mathcal{A}}_{\mathcal{B}} & \cos \theta^{\mathcal{A}}_{\mathcal{B}} & 0\\
		0 & 0 & 1
	\end{bmatrix} \triangleq
	\bm{R}\!\left(\cos\theta^{\mathcal{A}}_{\mathcal{B}}, \sin\theta^{\mathcal{A}}_{\mathcal{B}}\right).$
	For ease of reading, we list the main symbols in Appendix. 
	
	\section{Preliminaries and problem formulation}
	\subsection{Robot model}
	
	\begin{figure}[!t]\centering
		\includegraphics[scale=1.066]{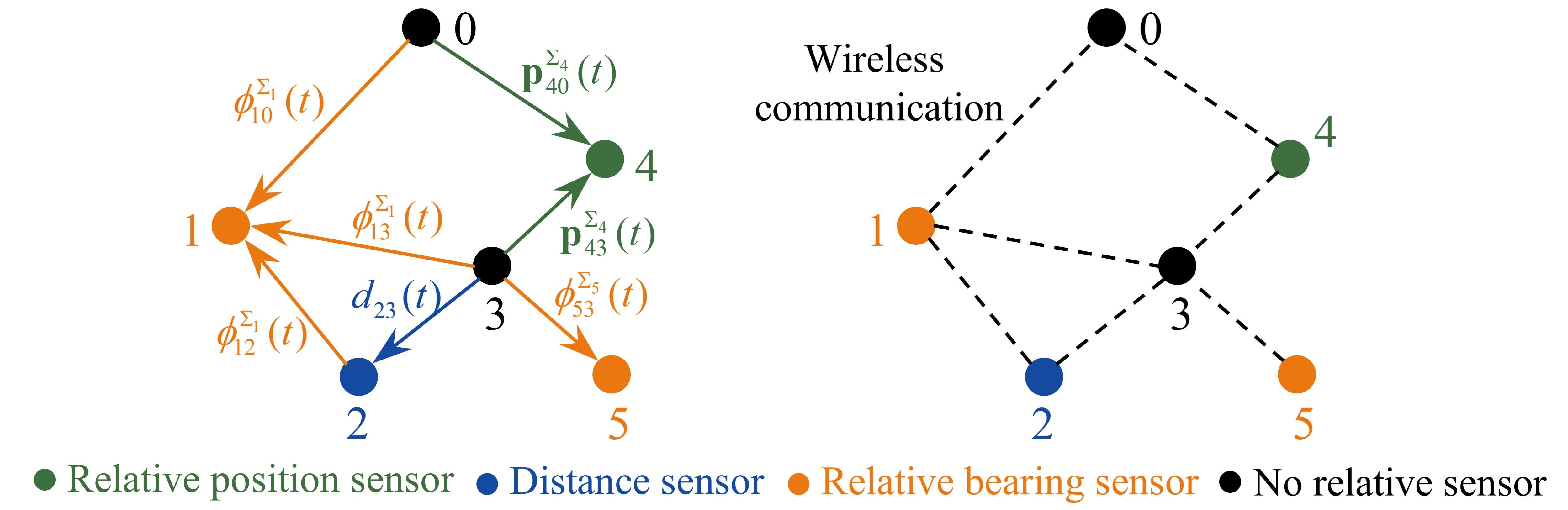}
		\caption{An example of measurement topology and communication topology. Left: measurement topology. Right: communication topology.}\centering
		\label{topology}
		\vspace{-0.5cm}
	\end{figure}
	
	Consider $N+1$ mobile robots in $\mathbb{R}^3$ modeled by 
	\blue{\begin{align}\label{robot}
			\dot{\bm{z}}^{\mathcal{O}_{i}}_{i}(t)
			= \bm{v}_{i}^{\mathcal{O}_{i}}(t),
			\qquad
			\dot{\theta}_{\Sigma_{i}}^{\mathcal{O}_{i}}(t)
			= w_i(t),
		\end{align}
		where $\bm{z}^{\mathcal{O}_{i}}_{i}$ and $\theta_{\Sigma_{i}}^{\mathcal{O}_{i}}$ are the translational displacement and yaw angle of $\Sigma_i$ relative to $\mathcal{O}_i$, respectively.
		The quantities $\bm{v}_{i}^{\mathcal{O}_i}(t)\in\mathbb{R}^{3}$ and $w_i(t)$ denote the translational velocity expressed in $\mathcal{O}_i$ and the yaw rate, respectively.}
	$v_{\mathrm{max}}$ and $w_{\mathrm{max}}$ denote the upper bounds of the linear velocity and yaw rate of robot $i$, i.e., $\lvert w_{i} \rvert \leq w_{\mathrm{max}}$ and $\lVert \bm{v}_{i}^{\mathcal{O}_{i}} \rVert \leq v_{\mathrm{max}}$.
	Each robot is equipped with onboard odometry to measure
	$\bm{z}_{i}^{\mathcal{O}_i}(t)$ and $\theta_{\Sigma_i}^{\mathcal{O}_i}(t)$, with
	$\bm{z}_{i}^{\mathcal{O}_i}(t_0)=\bm{0}$ and $\theta_{\Sigma_i}^{\mathcal{O}_i}(t_0)=0$.
	Consequently, the rotation matrix $\bm{R}^{\mathcal{O}_{i}}_{\Sigma_{i}} = \bm{R}\!\left(\cos\theta_{\Sigma_{i}}^{\mathcal{O}_{i}}, \sin\theta_{\Sigma_{i}}^{\mathcal{O}_{i}}\right)$ between $\Sigma_i$ and $\mathcal{O}_i$ is also available from odometry.
	
	\subsection{Measurement topology and communication topology} 
	To formalize the relative geometry between robots $i$ and $j$, define $\bm{p}^{\mathcal{O}_{i}}_{ij}(t_{0})$ as the position of robot $i$ relative to $j$, expressed in $\mathcal O_i$ and $\bm{R}_{\mathcal{O}_{i}} ^{\mathcal{O}_{j}} :=  \bm{R}\!\left(\cos\theta^{\mathcal{O}_{j}}_{\mathcal{O}_{i}}, \sin\theta^{\mathcal{O}_{j}}_{\mathcal{O}_{i}}\right)$ as the rotation matrix from $\mathcal{O}_{i}$ to $\mathcal{O}_{j}$,
	where $\theta^{\mathcal{O}_{j}}_{\mathcal{O}_{i}}$ is the constant initial relative yaw angle between robot $i$ and $j$.
	Both $\bm{p}^{\mathcal{O}_{i}}_{ij}(t_{0})$ and $\bm{R}_{\mathcal{O}_{i}} ^{\mathcal{O}_{j}}$ are unknown.
	With the odometer, the following equations hold:
	\begin{align}
		\notag
		\label{initial}
		\bm{p}^{\Sigma_{i}}_{ij}(t) &= \bm{R}^{\mathcal{O}_{i}}_{\Sigma_{i}}(t)^T \left( \bm{p}^{\mathcal{O}_{i}}_{ij}(t_{0}) + \bm{z}^{\mathcal{O}_{i}}_{i}(t) - \left( \bm{R}_{\mathcal{O}_{i}}^{\mathcal{O}_{j}} \right)^T \bm{z}^{\mathcal{O}_{j}}_{j}(t) \right), \\
		\bm{R}_{\Sigma_{i}}^{\Sigma_{j}} &= \bm{R}_{\Sigma_{j}}^{\mathcal{O}_{j}}(t)^{T} \bm{R}^{\mathcal{O}_{j}}_{\mathcal{O}_{i}} \bm{R}_{\Sigma_{i}}^{\mathcal{O}_{i}}(t),
	\end{align}
	where $\bm{p}^{\Sigma_{i}}_{ij}(t)$ denotes the real-time relative position between robots $i$ and $j$ in frame $\Sigma_{i}$. 
	In a heterogeneous measurement swarm, some robots are equipped with relative sensors, while others rely solely on odometry. Let $m_{ij}(t_k)$ denote the measurement obtained by robot $i$ with respect to robot $j$ at time instant $t_k$, which satisfies
	\begin{align}
		\notag
		m_{ij} \in \{ \bm{p}^{\Sigma_{i}}_{ij}, 
		d_{ij} = \lVert \bm{p}^{\Sigma_{i}}_{ij} \rVert,
		\bm{\phi}^{\Sigma_{i}}_{ij}
		= \frac{\bm{p}^{\Sigma_{i}}_{ij}}{\lVert \bm{p}^{\Sigma_{i}}_{ij} \rVert}, \varnothing \},
	\end{align} 
	where $\bm{p}^{\Sigma_{i}}_{ij}$, $d_{ij}$, and $\bm{\phi}^{\Sigma_{i}}_{ij}$
	denote the relative position, distance, and bearing of robot $i$ with respect to
	robot $j$, expressed in frame $\Sigma_{i}$ and obtained using relative sensors.
	The symbol $\varnothing$ indicates that robot $i$ is not equipped with any
	relative sensors and relies solely on odometry
	(i.e., robot~3 in Fig.~\ref{topology}).
	
	\begin{figure}[!t]\centering
		\includegraphics[scale=1.066]{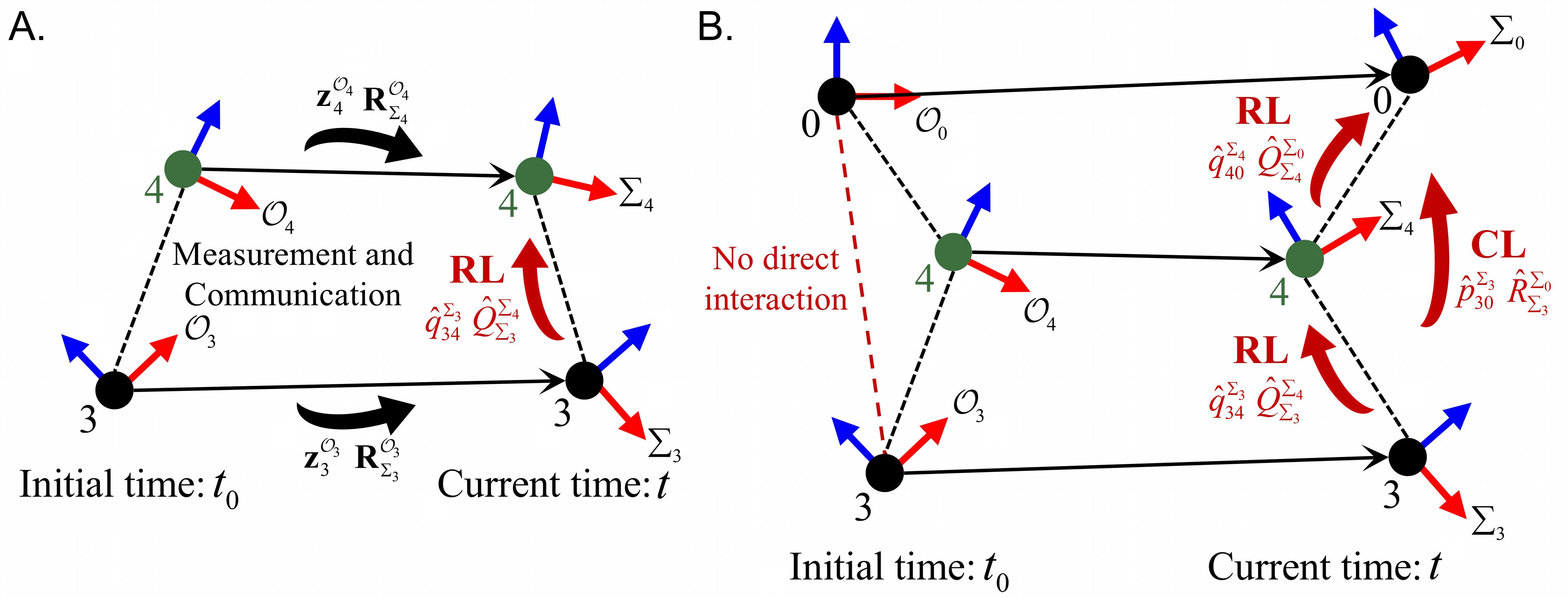}
		\caption{\blue{Relative localization and cooperative localization in distributed measurement and communication topology. A: Relative localization. B: Cooperative localization.}}\centering
		\label{rl_cl}
		\vspace{-0.5cm}
	\end{figure}
	
	\textbf{Measurement topology}: The \emph{directed measurement topology} is described as $\mathcal{G}_{m}=(\mathcal{V},\mathcal{E}_{m})$. 
	Each node $\mathcal{V}_{i}\in \mathcal{V}$ represents a robot. 
	A directed edge $(i,j) \in \mathcal{E}_{m}$ from robot $j$ to robot $i$
	indicates that robot $i$ can measure one of the following relative quantities
	with respect to robot $j$, i.e. $m_{ij}(t_{k}) \in \{ \bm{p}^{\Sigma_{i}}_{ij}(t_{k}), 
	d_{ij}(t_{k}), 
	\bm{\phi}^{\Sigma_{i}}_{ij}(t_{k})\}$.
	Robots $i$ and $j$ are defined as neighboring robots if
	$(i,j) \in \mathcal{E}_{m}$ or $(j,i) \in \mathcal{E}_{m}$.
	
	\textbf{Communication topology}: In this paper, neighboring robots are assumed to be able to exchange information through an undirected communication network.
	This assumption is physically reasonable and widely adopted in CL studies \cite{zheng2025Auto,xie2025tro}, since wireless communication is typically symmetric and less susceptible to environmental interference than directional sensors such as cameras.
	Accordingly, the \emph{undirected communication topology} is modeled as $\mathcal{G}_{c} = (\mathcal{V}, \mathcal{E}_{c})$, where $(i,j) \in \mathcal{E}_{c}$ if and only if $(i,j) \in \mathcal{E}_{m}$ or $(j,i) \in \mathcal{E}_{m}$.
	An example of the directed measurement topology and its corresponding
	communication topology is shown in Fig.~\ref{topology}.
	
	Let robot~$0$ serve as the anchor and provide a common reference frame for the swarm. 
	Let $F=\{1,\ldots,N\}$ denote the index set of the non-anchor robots. 
	For $i,j\in\{0\}\cup F$, set $a_{ij}=a_{ji}=1$ if $i\ne j$ and robots $i$ and $j$ are adjacent in $\mathcal G_c$, and set $a_{ij}=0$ otherwise. Let $\mathcal L\in\mathbb R^{N\times N}$ denote the Laplacian matrix of $\mathcal G_c[F]$, the subgraph induced by $F$. For each $i\in F$, define its set of non-anchor neighbors as $\mathcal N_i=\{j\in F:a_{ij}=1\}$ and its anchor adjacency indicator as $\mu_i=a_{i0}$. Robot~$i$ is called an \emph{informed robot} if $\mu_i=1$, meaning that it communicates directly with the anchor. Define $\mathcal B=\operatorname{diag}(\mu_1,\ldots,\mu_N)$.
	To achieve the CL task, we introduce the following assumption:
	\begin{assumption}
		\label{top}
		The directed measurement graph $\mathcal{G}_{m}$ is weakly connected, i.e.,
		the undirected graph $\mathcal{G}_{c}$ is connected.
	\end{assumption}
	
	
	\subsection{Objective of RL and CL}
	This section explains the objectives of relative localization and cooperative localization, as well as their interrelationships.
	
	\blue{\textbf{Relative Localization}: For neighboring robots $i$ and $j$
		in $\mathcal{G}_{c}$, such as robots $3$ and $4$ in
		Fig.\ref{topology}, RL aims to estimate their real-time relative
		position $\bm{p}_{ij}^{\Sigma_i}$ and orientation
		$\bm{R}_{\Sigma_i}^{\Sigma_j}$ using the available relative
		measurements $m_{ij}$ and $m_{ji}$, odometry data, and pairwise
		communication. 
		The corresponding estimates are denoted by
		$\hat{\bm{q}}_{ij}^{\Sigma_i}$ and
		$\hat{\bm{Q}}_{\Sigma_i}^{\Sigma_j}$,
		respectively, as shown in Fig.\ref{rl_cl}(A).}
	
	\blue{\textbf{Cooperative Localization}: Robot $0$ serves as the anchor
		and provides a common pose reference for the swarm.
		Building on pairwise RL estimates, CL uses distributed
		information exchange to estimate each robot's real-time
		position $\bm{p}_{i0}^{\Sigma_i}$ and orientation
		$\bm{R}_{\Sigma_i}^{\Sigma_0}$ relative to the anchor.
		The corresponding estimates
		$\hat{\bm{p}}_{i0}^{\Sigma_i}$ and
		$\hat{\bm{R}}_{\Sigma_i}^{\Sigma_0}$
		are designed to converge asymptotically to their true values.
		For example, robot $3$ can localize itself relative to robot $0$
		without a direct measurement or communication link to it,
		as illustrated in Figs.~\ref{topology} and~\ref{rl_cl}(B).}
	\blue{\begin{remark}
			In the absence of an external global reference frame, the anchor robot provides a common reference for the swarm.
			RL estimates the relative pose between any pair of neighboring robots in the measurement and communication topologies, including pairs involving the anchor.
			However, these distributed topologies do not allow every robot to perform RL directly with the anchor.
			CL in this paper therefore refers to the process by which robots determine their relative poses with respect to the anchor through distributed consensus.
	\end{remark}}
	
	To accomplish the CL task, the following lemma is useful.
	\begin{lemma} \cite{Biggs1993book}
		\label{connected}
		If the undirected graph $\mathcal{G}_{c}$ is connected and at least one robot has direct access to the anchor robot $0$, then the matrix $(\mathcal{L} + \mathcal{B})$ is positive definite.
	\end{lemma}
	
	
	\section{Relative localization} 
	In this section, we propose a unified adaptive estimator to achieve pairwise RL between arbitrary neighboring robots.
	
	\subsection{Unified relative localization model}
	\label{section model}
	
	Recall (\ref{initial}), since displacements $\bm{z}^{\mathcal{O}_{i}}_{i}$, $\bm{R}^{\mathcal{O}_{i}}_{\Sigma_{i}}$, $\bm{z}^{\mathcal{O}_{j}}_{j}$ and  $\bm{R}^{\mathcal{O}_{j}}_{\Sigma_{j}}$ are available via odometry and communication. 
	To estimate the real-time relative pose $\bm{p}^{\Sigma_{i}}_{ij}$ and $\bm{R}_{\Sigma_{i}}^{\Sigma_{j}}$, the initial relative position $\bm{p}^{\mathcal{O}_{i}}_{ij}(t_{0})$ and initial rotation matrix $\bm{R}^{\mathcal{O}_{j}}_{\mathcal{O}_{i}}$ from $\mathcal{O}_{i}$ to $\mathcal{O}_{j}$ should be estimated. 
	In other words, each robot needs to estimate the initial relative state $\bm{\Theta}^{*}_{ij}(t_{0}) = [\bm{p}^{\mathcal{O}_{i}}_{ij}(t_{0})^T, \cos \theta^{\mathcal{O}_{j}}_{\mathcal{O}_{i}}, \sin \theta^{\mathcal{O}_{j}}_{\mathcal{O}_{i}}]^T$ with the measurements $m_{ij}(t_{k})$ and $m_{ji}(t_{k})$ at time instant $t_{k}$, respectively.
	Here, the sampling time instant $t_{k}$ is defined as $t_{k} \triangleq t_{0} + k \Delta t, k \in \mathbb{N}$.
	A portion of $\bm{\Theta}_{ij}^{*}(t_{0})$ can be directly calculated from the initial relative measurements $m_{ij}(t_{0})$ and $m_{ji}(t_{0})$. 
	The remaining part, denoted as the unknown state $\bm{\Theta}_{ij}(t_{0})$, is the key to fully determining the initial relative state. Specifically, $\bm{\Theta}_{ij}^{*}(t_{0})$ is uniquely computable if and only if $\bm{\Theta}_{ij}(t_{0})$ can be estimated. This relationship is governed by the following equation:
	\begin{align}
		\label{relationship}
		\bm{\Theta}_{ij}^{*}(t_{0}) = H_{ij}(\bm{\Theta}_{ij}(t_{0}), m_{ij}(t_{0}), m_{ji}(t_{0})),
	\end{align}
	where $H_{ij}(.)$ is a known function and is globally Lipschitz continuous with respect to $\bm{\Theta}_{ij}(t_{0})$.
	To estimate the unknown parameter $\bm{\Theta}_{ij}(t_{0})$, we aim to construct the following linear observation equation for an arbitrary pair of neighbors.
	\begin{align}
		\label{rl_model}
		y_{ij}(t_{k}) =  \bm{\Phi}_{ij}(t_{k})^{T} \bm{\Theta}_{ij}(t_{0}),
	\end{align}
	where $y_{ij}(t_{k})$ and $\bm{\Phi}_{ij}(t_{k})$ are obtainable and can be calculated with the relative sensor and odometer measurements.
	For neighboring robots equipped with different sensing modalities, the explicit mathematical form of \eqref{rl_model} varies accordingly. Details are provided in the next section.
	
	\subsection{Observation equation under different sensor combinations}
	\begin{figure}[!t]\centering
		\includegraphics[scale=1.066]{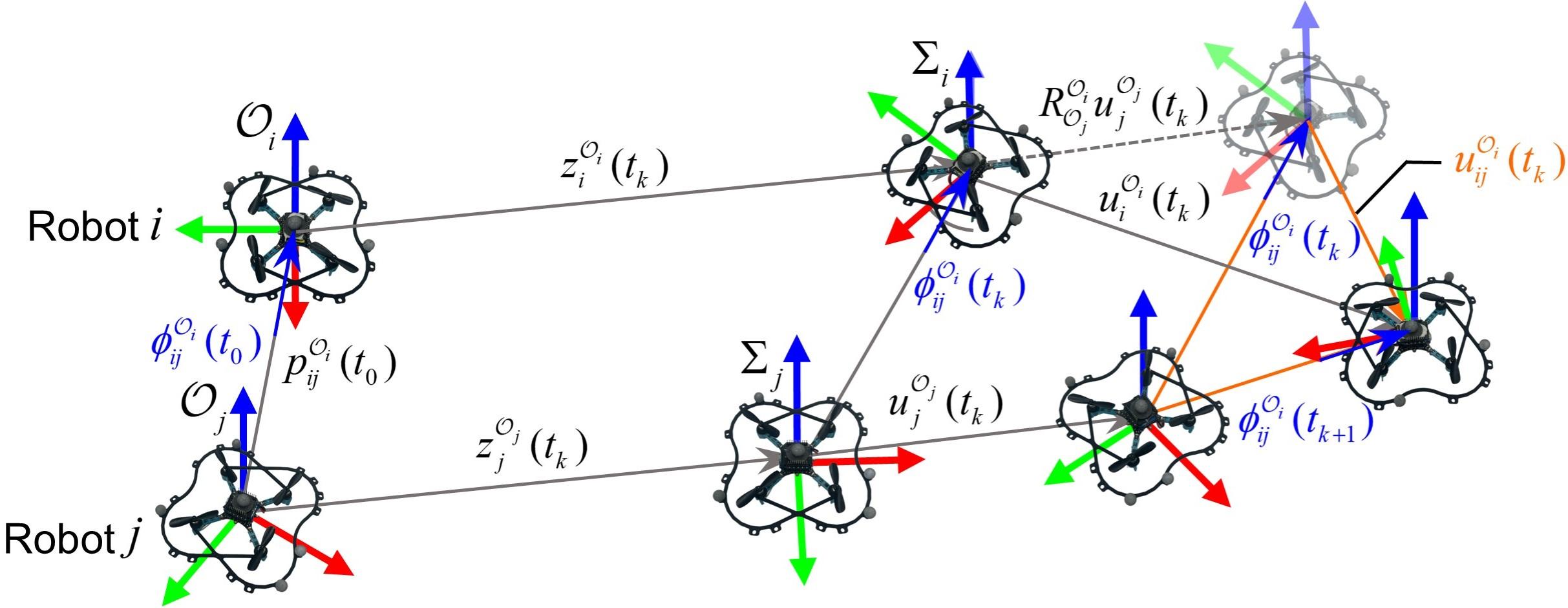}
		\caption{Geometric relationship between the displacement and unidirectional bearing measurements of the two robots.}\centering
		\label{rl_figure}
		\vspace{-0.5cm}
	\end{figure}
	
	The relative measurements $m_{ij}$ and $m_{ji}$ between neighboring robots $i$ and $j$ may be different. 
	Each dual-sensor configuration provides the measurements required by at least one of the three single-sensor cases considered below. 
	Consequently, the measurement data observed in those single-sensor cases are contained within the richer measurement sets of the dual-sensor cases. Therefore, analyzing the observation equations for those three single-sensor cases provides a basis for handling all remaining configurations.
	Without loss of generality, let robot $j$ be a sensorless robot. 
	The relative measurement of robot $i$ for $j$ at time instant $t_{k}$ is $m_{ij}(t_{k})$.
	
	\textbf{Case 1}: $m_{ij} = \bm{\phi}_{ij}^{\Sigma_{i}}$ and $m_{ji} = \varnothing$
	
	With the initial bearing information $\bm{\phi}_{ij}^{\Sigma_{i}}(t_{0})$ obtainable, it can be easily seen that the unknown parameter $\bm{\Theta}_{ij}(t_{0})$ is $\bm{\Theta}_{ij}(t_{0}) = [d_{ij}(t_{0}), \cos \theta^{\mathcal{O}_{j}}_{\mathcal{O}_{i}}, \sin \theta^{\mathcal{O}_{j}}_{\mathcal{O}_{i}} ]^T$. 
	The initial relative state $\bm{\Theta}_{ij}^{*}(t_{0})$ can be calculated as $\bm{\Theta}_{ij}^{*}(t_{0}) =  [\left(\bm{\phi}_{ij}^{\mathcal{O}_{i}}(t_{0})d_{ij}(t_{0})\right)^T, \cos \theta^{\mathcal{O}_{j}}_{\mathcal{O}_{i}}, \sin \theta^{\mathcal{O}_{j}}_{\mathcal{O}_{i}} ]^{T}$, which is globally Lipschitz continuous with respect to $\bm{\Theta}_{ij}(t_{0})$.
	Next, we construct the observation equation (\ref{rl_model}) under the single bearing measurement case.
	Since $\bm{R}_{\Sigma_{i}}^{\mathcal{O}_{i}}$ is available, $\bm{\phi}_{ij}^{\mathcal{O}_{i}}(t_{k}) = \bm{R}_{\Sigma_{i}}^{\mathcal{O}_{i}} \bm{\phi}_{ij}^{\Sigma_{i}}(t_{k})$ is also available.
	Define the translational displacement of robot $i$ from $t_{k}$ to $t_{k+1}$ in its odometer frame $\mathcal{O}_{i}$ as $\bm{u}_{i}^{\mathcal{O}_{i}}$, as shown in Fig. \ref{rl_figure}. Define $\bm{u}_{ij}^{\mathcal{O}_{i}}(t_{k}) = \left( \bm{u}_{i}^{\mathcal{O}_{i}}(t_{k}) - \left( \bm{R}_{\mathcal{O}_{i}}^{\mathcal{O}_{j}} \right)^T \bm{u}_{j}^{\mathcal{O}_{j}}(t_{k})\right)$ as the relative translational displacement between robots $i$ and $j$ from $t_{k}$ to $t_{k+1}$, it has
	\begin{align}
		\notag
		d_{ij}(t_{k}) \left( \bm{\phi}_{ij}^{\mathcal{O}_{i}}(t_{k}) \times \bm{\phi}_{ij}^{\mathcal{O}_{i}}(t_{k+1}) \right) = - \bm{u}_{ij}^{\mathcal{O}_{i}}(t_{k}) \times \bm{\phi}_{ij}^{\mathcal{O}_{i}}(t_{k+1}),
	\end{align}
	Taking the inner product of each side with $\bm{\phi}_{ij}^{\mathcal{O}_{i}}(t_{k})$, it has
	\begin{align}
		\label{distance1}
		\notag
		&\bm{a}_{ij}(t_{k})^{T} d_{ij}(t_{k}) \bm{\phi}_{ij}^{\mathcal{O}_{i}}(t_{k}) + \left( \bm{u}_{i}^{\mathcal{O}_{i}}(t_{k}) \times \bm{\phi}_{ij}^{\mathcal{O}_{i}}(t_{k+1}) \right)^{T} \bm{\phi}_{ij}^{\mathcal{O}_{i}}(t_{k}) 
		\\
		&- \left( \left( \left(\bm{R}_{\mathcal{O}_{i}}^{\mathcal{O}_{j}} \right)^T \bm{u}_{j}^{\mathcal{O}_{j}}(t_{k}) \right)  \times \bm{\phi}_{ij}^{\mathcal{O}_{i}}(t_{k+1})\right)^{T} \bm{\phi}_{ij}^{\mathcal{O}_{i}}(t_{k}) = 0,
	\end{align}
	where $\bm{a}_{ij}(t_{k}) \triangleq \left( \bm{\phi}_{ij}^{\mathcal{O}_{i}}(t_{k}) \times \bm{\phi}_{ij}^{\mathcal{O}_{i}}(t_{k+1}) \right)$ is measurable.
	
	According to (\ref{initial}), $d_{ij}(t_{k})$ satisfies
	\begin{align}
		\label{distance}
		\bm{p}_{ij}^{\mathcal{O}_{i}}(t_{k}) = d_{ij}(t_{0}) \bm{\phi}_{ij}^{\mathcal{O}_{i}}(t_{0}) + \bm{z}^{\mathcal{O}_{i}}_{i}(t_{k}) - \left( \bm{R}_{\mathcal{O}_{i}}^{\mathcal{O}_{j}} \right)^T
		\bm{z}^{\mathcal{O}_{j}}_{j}(t_{k}),
	\end{align}
	where $\bm{p}_{ij}^{\mathcal{O}_{i}}(t_{k}) = d_{ij}(t_{k}) \bm{\phi}_{ij}^{\mathcal{O}_{i}}(t_{k})$. 
	
	Substituting Eq.(\ref{distance}) into Eq.(\ref{distance1}), it has
	\begin{align}
		\label{distance2}
		\notag
		&\bm{a}_{ij}(t_{k})^{T} \bm{z}^{\mathcal{O}_{i}}_{i}(t_{k}) + \left( \bm{u}_{i}^{\mathcal{O}_{i}}(t_{k}) \times \bm{\phi}_{ij}^{\mathcal{O}_{i}}(t_{k+1}) \right)^{T} \bm{\phi}_{ij}^{\mathcal{O}_{i}}(t_{k})  
		\\ 
		\notag
		= & -\bm{a}_{ij}(t_{k})^{T} d_{ij}(t_{0}) \bm{\phi}_{ij}^{\mathcal{O}_{i}}(t_{0}) + \bm{a}_{ij}(t_{k})^{T}\left( \bm{R}_{\mathcal{O}_{i}}^{\mathcal{O}_{j}} \right)^T\bm{z}^{\mathcal{O}_{j}}_{j}(t_{k})
		\\ 
		&+ \left( \left( \left( \bm{R}_{\mathcal{O}_{i}}^{\mathcal{O}_{j}} \right)^T \bm{u}_{j}^{\mathcal{O}_{j}}(t_{k}) \right)  \times \bm{\phi}_{ij}^{\mathcal{O}_{i}}(t_{k+1})\right)^{T} \bm{\phi}_{ij}^{\mathcal{O}_{i}}(t_{k}).
	\end{align}
	Define $\bm{a}_{ij} = [a_{ij,1}, a_{ij,2}, a_{ij,3}]^{T} \triangleq [\bm{a}_{ij,h}^T, a_{ij,3}]^T$. Similarly, take $\bm{z}_{j}^{\mathcal{O}_{j}} \triangleq [\left( \bm{z}_{j,h}^{\mathcal{O}_{j}} \right)^{T}, z_{j,3}^{\mathcal{O}_{j}}]^T$,
	$\bm{\phi}_{ij}^{\mathcal{O}_{i}} \triangleq [\left( \bm{\phi}_{ij,h}^{\mathcal{O}_{i}} \right)^T, \phi_{ij,3}^{\mathcal{O}_{i}}]^T$ 
	and $\bm{u}_{j}^{\mathcal{O}_{j}} \triangleq [\left( \bm{u}_{j,h}^{\mathcal{O}_{j}} \right)^{T}, u_{j,3}^{\mathcal{O}_{j}}]^T$, then one can check that 
	\begin{align}
		\label{caculate}
		&\bm{a}_{ij}(t_{k})^{T} \left(\bm{R}_{\mathcal{O}_{i}}^{\mathcal{O}_{j}} \right)^T \bm{z}^{\mathcal{O}_{j}}_{j}(t_{k}) = a_{ij,3}(t_{k}) z_{j,3}^{\mathcal{O}_{j}}(t_{k}) \\
		\notag
		& + [\bm{a}_{ij,h}(t_{k})^T \bm{z}_{j,h}^{\mathcal{O}_{j}}(t_{k}), -\bm{a}_{ij,h}(t_{k}) \times \bm{z}_{j,h}^{\mathcal{O}_{j}}(t_{k})] \begin{bmatrix}
			\cos \theta^{\mathcal{O}_{j}}_{\mathcal{O}_{i}} \\ \sin \theta^{\mathcal{O}_{j}}_{\mathcal{O}_{i}}
		\end{bmatrix} \\
		\notag
		&\left( \left( \left(\bm{R}_{\mathcal{O}_{i}}^{\mathcal{O}_{j}} \right)^T \bm{u}_{j}^{\mathcal{O}_{j}}(t_{k}) \right)  \times \bm{\phi}_{ij}^{\mathcal{O}_{i}}(t_{k+1})\right)^{T} \bm{\phi}_{ij}^{\mathcal{O}_{i}}(t_{k}) \\ 
		\notag
		& =[-\bm{u}_{j,h}^{\mathcal{O}_{j}}(t_{k})^T \bm{a}_{ij,h}(t_{k}), -\bm{u}_{j,h}^{\mathcal{O}_{j}}(t_{k}) \times \bm{a}_{ij,h}(t_{k})]\begin{bmatrix}
			\cos \theta^{\mathcal{O}_{j}}_{\mathcal{O}_{i}} \\ \sin \theta^{\mathcal{O}_{j}}_{\mathcal{O}_{i}}
		\end{bmatrix} \\
		\notag
		&-a_{ij,3}(t_{k}) u_{j,3}^{\mathcal{O}_{j}}(t_{k}).
	\end{align}
	Substituting Eq.(\ref{caculate}) to Eq.(\ref{distance2}), one can see 
	\begin{align}
		\label{model} 
		&\bar{y}_{ij}(t_{k}) \triangleq \left( \bm{u}_{i}^{\mathcal{O}_{i}}(t_{k}) \times \bm{\phi}_{ij}^{\mathcal{O}_{i}}(t_{k+1}) \right)^{T} \bm{\phi}_{ij}^{\mathcal{O}_{i}}(t_{k}) 
		\\
		\notag
		& + \bm{a}_{ij}(t_{k})^{T} \bm{z}^{\mathcal{O}_{i}}_{i}(t_{k}) - a_{ij,3}(t_{k}) z_{j,3}^{\mathcal{O}_{j}}(t_{k}) + a_{ij,3}(t_{k}) u_{j,3}^{\mathcal{O}_{j}} (t_{k})
		\\
		\notag 
		= & [\psi_{ij,1}(t_{k}), \psi_{ij,2}(t_{k}), \psi_{ij,3}(t_{k})] \bm{\Theta}_{ij}(t_{0})
		\triangleq \bm{\Psi}_{ij}(t_{k})^T \bm{\Theta}_{ij}(t_{0}), 
	\end{align}
	where 
	$\psi_{ij,1}(t_{k}) =  -\bm{a}_{ij}(t_{k})^{T} \bm{\phi}_{ij}^{\mathcal{O}_{i}}(t_{0})$, $\psi_{ij,2}(t_{k}) =\bm{a}_{ij,h}(t_{k})^T \bm{z}_{j,h}^{\mathcal{O}_{j}}(t_{k})-\bm{u}_{j,h}^{\mathcal{O}_{j}}(t_{k})^T \bm{a}_{ij,h}(t_{k})$ and $\psi_{ij,3}(t_{k}) =-\bm{u}_{j,h}^{\mathcal{O}_{j}}(t_{k}) \times \bm{a}_{ij,h}(t_{k})-\bm{a}_{ij,h}(t_{k}) \times \bm{z}_{j,h}^{\mathcal{O}_{j}}(t_{k})$.
	Note that for vector $\bm{x}\in \mathbb{R}^{2}$, $\bm{y}\in \mathbb{R}^{2}$, $\bm{x} \times \bm{y} = x[2]y[1]-x[1]y[2]$.
	In Eq.(\ref{model}), one can see that both $\bar{y}_{ij}(t_{k})$ and $\bm{\Psi}_{ij}(t_{k})$ are available to robot $i$ at time instant $t_{k+1}$. 
	The RL problem between robots $i$ and $j$ is transformed into a parameter estimation problem for $\bm{\Theta}_{ij}(t_{0})$. 
	Choose parameter $\kappa > 0$, 
	define $y_{ij}(t_{k}) \triangleq \frac{\bar{y}_{ij}(t_{k})}{\lVert \bm{\Psi}_{ij}(t_{k})\rVert + \kappa }$ and normalize Eq.(\ref{model}) to obtain
	\begin{align}
		y_{ij}(t_{k}) = \frac{\bm{\Psi}_{ij}(t_{k})^{T}}{\lVert \bm{\Psi}_{ij}(t_{k}) \rVert + \kappa } \bm{\Theta}_{ij}(t_{0}) \triangleq \bm{\Phi}_{ij}(t_{k})^{T} \bm{\Theta}_{ij}(t_{0}),
	\end{align}
	which has the same structure as the unified linear model (\ref{rl_model})
	
	\textbf{Case 2}: $m_{ij} = \bm{p}_{ij}^{\Sigma_{i}}$ and $m_{ji} = \varnothing$
	
	Similarly to the bearing case, with the relative position $\bm{p}_{ij}^{\mathcal{O}_{i}}$ in frame $\mathcal{O}_{i}$ available, the unknown parameter that needs to be estimated is $\bm{\Theta}_{ij}(t_{0}) = [\cos \theta^{\mathcal{O}_{j}}_{\mathcal{O}_{i}}, \sin \theta^{\mathcal{O}_{j}}_{\mathcal{O}_{i}}]^T$.
	The initial relative state $\bm{\Theta}_{ij}^{*}(t_{0})$ can be directly represented as $\bm{\Theta}_{ij}^{*}(t_{0}) =  [\bm{p}_{ij}^{\mathcal{O}_{i}}(t_{0})^T, \cos \theta^{\mathcal{O}_{j}}_{\mathcal{O}_{i}}, \sin \theta^{\mathcal{O}_{j}}_{\mathcal{O}_{i}} ]^{T}$, which is globally Lipschitz continuous with respect to $\bm{\Theta}_{ij}(t_{0})$.
	Since $\bm{R}_{\Sigma_{i}}^{\mathcal{O}_{i}}$ is available, $ \bm{p}_{ij}^{\mathcal{O}_{i}}(t_{k})= \bm{R}_{\Sigma_{i}}^{\mathcal{O}_{i}}(t_{k}) \bm{p}_{ij}^{\Sigma_{i}}(t_{k})$ is also available. 
	One can see the following relationship is established:
	\begin{align}
		\notag
		\left(\bm{R}_{\mathcal{O}_{i}}^{\mathcal{O}_{j}}\right)^T \bm{u}_{j}^{\mathcal{O}_{j}}(t_{k}) = \left(\bm{p}_{ij}^{\mathcal{O}_{i}}(t_{k}) +  \bm{u}_{i}^{\mathcal{O}_{i}}(t_{k}) - \bm{p}_{ij}^{\mathcal{O}_{i}}(t_{k+1})\right).
	\end{align}
	Let $\bm{p}_{ij}^{\mathcal{O}_{i}} \triangleq [\left(\bm{p}_{ij,h}^{\mathcal{O}_{i}}\right)^{T}, p_{ij,3}^{\mathcal{O}_{i}}]^{T} $ and $\bm{u}_{j}^{\mathcal{O}_{j}} \triangleq [\left(\bm{u}_{j,h}^{\mathcal{O}_{j}}\right)^{T}, u_{j,3}^{\mathcal{O}_{j}}]^{T} $.
	Choose parameter $\kappa>0$, one can check that
	\begin{align}
		y_{ij}(t_{k}) = &\frac{\bm{\Psi}_{ij}(t_{k})^{T}}{\lVert \bm{\Psi}_{ij}(t_{k}) \rVert + \kappa } \bm{\Theta}_{ij}(t_{0}) \triangleq \bm{\Phi}_{ij}(t_{k})^{T} \bm{\Theta}_{ij}(t_{0}), \\
		\notag
		\bm{\Psi}_{ij}(t_{k}) \triangleq &[\bm{p}_{ij,h}^{\mathcal{O}_{i}}(t_{k})^T \bm{u}_{j,h}^{\mathcal{O}_{j}}(t_{k}), -\bm{p}_{ij,h}^{\mathcal{O}_{i}}(t_{k}) \times \bm{u}_{j,h}^{\mathcal{O}_{j}}(t_{k})]^{T}
	\end{align}
	where $\bar{y}_{ij}(t_{k}) \triangleq \bm{p}_{ij}^{\mathcal{O}_{i}}(t_{k})^{T}\left(\bm{p}_{ij}^{\mathcal{O}_{i}}(t_{k}) +  \bm{u}_{i}^{\mathcal{O}_{i}}(t_{k}) - \bm{p}_{ij}^{\mathcal{O}_{i}}(t_{k+1})\right)
	-p_{ij,3}^{\mathcal{O}_{i}}(t_{k}) u_{j,3}^{\mathcal{O}_{j}}(t_{k})$ and $y_{ij}(t_{k}) \triangleq \frac{\bar{y}_{ij}(t_{k})}{\lVert \bm{\Psi}_{ij}(t_{k}) \rVert + \kappa }$, 
	which has the same structure with (\ref{rl_model})
	
	\textbf{Case 3}: $m_{ij} = d_{ij}$ and $m_{ji} = \varnothing$
	
	Since distance measurement $d_{ij}$ is a scalar, one neighbor can measure and inform the other through communication.
	
	For this sensing configuration, distance and odometry measurements can also be used to construct a linear observation equation of the form (\ref{rl_model}) in the constant parameter vector $\mathbf{\Theta}_{ij}
	=
	\left[
	\bm{p}^{\mathcal{O}_{i}}_{ij}(t_{0})^{T},
	\left(
	\bm{R}^{\mathcal{O}_{j,h}}_{\mathcal{O}_{i,h}}
	\bm{p}_{ij,h}^{\mathcal{O}_{i}}(t_{0})
	\right)^T,
	\cos \theta_{\mathcal{O}_{i}}^{\mathcal{O}_{j}},
	\sin \theta_{\mathcal{O}_{i}}^{\mathcal{O}_{j}}
	\right]^T$, 
	where
	$\bm{R}^{\mathcal{O}_{j,h}}_{\mathcal{O}_{i,h}}
	=
	\begin{bmatrix}
		\cos \theta_{\mathcal{O}_{i}}^{\mathcal{O}_{j}}
		&
		-\sin \theta_{\mathcal{O}_{i}}^{\mathcal{O}_{j}}\\
		\sin \theta_{\mathcal{O}_{i}}^{\mathcal{O}_{j}}
		&
		\cos \theta_{\mathcal{O}_{i}}^{\mathcal{O}_{j}}
	\end{bmatrix}$
	is the horizontal rotation matrix from
	$\mathcal{O}_{i,h}$ to $\mathcal{O}_{j,h}$.
	The observation equation for this situation can be found in \cite{ze2024relative} and is not elaborated here.
	The relationship between $\bm{\Theta}_{ij}(t_{0})$ and $\bm{\Theta}_{ij}^{*}(t_{0})$ can be described by \eqref{relationship}. 
	
	For the other 6 measurement cases, at least one unified observation equation can be constructed based on the above 3 basic scenarios. 
	In the next section, an adaptive observer is designed to accomplish the RL task between neighbors.
	
	\subsection{Adaptive relative localization}
	\label{RL}
	For robot $i$, define $\hat{\bm{\Theta}}_{ij}$ as an estimation of parameter $\bm{\Theta}_{ij}(t_{0})$ and $\tilde{\bm{\Theta}}_{ij} = \hat{\bm{\Theta}}_{ij} -\bm{\Theta}_{ij}(t_{0})$ as the estimation error. 
	Define the following innovation $\nu_{ij}$
	\begin{align}
		\notag
		\nu_{ij}(t_{k};t_{m}) = \hat{\bm{\Theta}}_{ij}(t_{k})^{T} \bm{\Phi}_{ij}(t_{m}) - y_{ij}(t_{m}).
	\end{align}
	Based on the relative localization model built in Eq.(\ref{rl_model}), design the  following adaptive RL estimator $\hat{\bm{\Theta}}_{ij}$ to estimate $\bm{\Theta}_{ij}(t_{0})$
	\begin{align}
		\label{relative estimator}
		\notag
		&\hat{\mathbf{\Theta}}_{ij}(t) = \hat{\mathbf{\Theta}}_{ij}(t_{k}), \quad t_{k} < t < t_{k+1} \\ \notag
		&\hat{\mathbf{\Theta}}_{ij}(t_{k+1}) = \hat{\mathbf{\Theta}}_{ij}(t_{k}) - \zeta_{ij}(t_{k}) \sum_{m=1}^{\varsigma} \mathbf{\Phi}_{ij}(t_{m}) \nu_{ij}(t_{k};t_{m}) \\ 
		&- \zeta_{ij}(t_{k}) \mathbf{\Phi}_{ij}(t_{k}) \nu_{ij}(t_{k};t_{k}),
	\end{align}
	where $\zeta$ is designed as $\zeta_{ij}(t_{k}) = \frac{\lambda_{\mathrm{min}}(\mathcal{S}_{ij})}{\left(\lambda_{\mathrm{max}}(\mathcal{U}_{ij}(t_{k})) + \lambda_{\mathrm{max}}(\mathcal{S}_{ij})\right)^2}$, 
	with current data matrix $\mathcal{U}_{ij}(t_{k}) = \bm{\Phi}_{ij}(t_{k})\bm{\Phi}_{ij}(t_{k})^{T}$.
	\blue{The recorded data matrix $\mathcal{S}_{ij}$ is initialized to zero and set to $\sum_{m=1}^{\varsigma}\bm{\Phi}_{ij}(t_m)\bm{\Phi}_{ij}(t_m)^T$ once all $\varsigma$ samples have been collected.}
	Design the initial relative pose estimator as 
	\begin{align}
		\label{rl_estimator}
		\hat{\bm{\Theta}}^{*}_{ij}(t)=H_{ij}(\hat{\bm{\Theta}}_{ij}(t), m_{ij}(t_{0}), m_{ji}(t_{0})).
	\end{align}
	Recall that $m_{ij}(t_{0})$ and $ m_{ji}(t_{0})$ are initial relative measurement information in frames $\mathcal{O}_{i}$ and $\mathcal{O}_{j}$. 
	For robot $i$, redefine the RL estimator (\ref{rl_estimator}) as $\hat{\bm{\Theta}}^{*}_{ij}(t) \triangleq [\hat{\bm{q}}_{ij}^{\mathcal{O}_{i}}(t)^T, \hat{c}_{\mathcal{O}_{i}}^{\mathcal{O}_{j}}(t), \hat{s}_{\mathcal{O}_{i}}^{\mathcal{O}_{j}}(t)]^{T}$.
	
	\blue{Define $\hat{\bm{q}}_{ij}^{\mathcal{O}_{i}}(t)$ and 	$\hat{\bm{Q}}_{\mathcal{O}_{i}}^{\mathcal{O}_{j}}(t)
		=
		\bm{R}\left(
		\hat{c}_{\mathcal{O}_{i}}^{\mathcal{O}_{j}}(t),
		\hat{s}_{\mathcal{O}_{i}}^{\mathcal{O}_{j}}(t)
		\right)$ to estimate the initial relative pose $\bm{p}_{ij}^{\mathcal{O}_{i}}(t_{0})$ and  $\bm{R}_{\mathcal{O}_{i}}^{\mathcal{O}_{j}}$ between $i$ and $j$, respectively.}
	Define the pose estimation error as $\tilde{\bm{\Theta}}^{*}_{ij}(t) = \hat{\bm{\Theta}}^{*}_{ij} - \bm{\Theta}^{*}_{ij} \triangleq [\tilde{\bm{q}}_{ij}^{\mathcal{O}_{i}}(t)^T, \tilde{c}_{\mathcal{O}_{i}}^{\mathcal{O}_{j}}(t), \tilde{s}_{\mathcal{O}_{i}}^{\mathcal{O}_{j}}(t)]^{T}$.
	The convergence of $\tilde{\bm{\Theta}}^{*}_{ij}$ can be guaranteed by the following full rank condition on the recorded data matrix $\mathcal{S}_{ij}$.
	\begin{assumption}
		\label{a1}
		The recorded data matrix $\mathcal{S}_{ij}$ is full rank.
	\end{assumption}
	
	Assumption \ref{a1} indicates that the data matrix $\mathcal{S}_{ij}$ contains sufficient information to accomplish the RL task. This is a common assumption in adaptive parameter estimation methods. 
	The motion of each robot $i$ can be designed so that the data matrix $\mathcal{S}_{ij}$ has full rank.
	
	\begin{theorem}
		\label{t1}
		Suppose that the recorded data matrix $\mathcal{S}_{ij}$ satisfies Assumption~\ref{a1} at $t_{k_c}$.
		Then the parameter estimation error $\tilde{\bm{\Theta}}^{*}_{ij}(t)$ converges exponentially to zero for $t\geq t_{k_c}$.
	\end{theorem}
	
	\begin{proof}
		Choose a Lyapunov candidate as $V_{ij}(t_{k}) = \tilde{\mathbf{\Theta}}_{ij}(t_{k})^T\tilde{\mathbf{\Theta}}_{ij}(t_{k})$. Then $\Delta V_{ij} = V_{ij}(t_{k+1}) - V_{ij}(t_{k})$ satisfies
		\begin{align}
			\label{d_V}
			\Delta V_{ij} \leq &\zeta^2 \left( \lambda_{\max}(\mathcal{U}_{ij}(t_{k})) + \lambda_{\max}(\mathcal{S}_{ij}) \right)^2 \lVert \tilde{\mathbf{\Theta}}_{ij}(t_{k}) \rVert^2 \\
			\notag
			&- 2 \zeta \lambda_{\min}(\mathcal{S}_{ij}) \lVert \tilde{\mathbf{\Theta}}_{ij}(t_{k}) \rVert^2.
		\end{align}
		Substituting $\zeta$ into Eq.(\ref{d_V}), one can see $\Delta V_{ij} \leq \frac{-\lambda_{\min}(\mathcal{S}_{ij})^2}{(\lambda_{\max}(\mathcal{U}_{ij}(t_{k})) + \lambda_{\max}(\mathcal{S}_{ij})) ^2} \lVert \tilde{\mathbf{\Theta}}_{ij}(t_{k}) \rVert^2 $.
		Since $\lambda_{\max}(\mathcal{U}_{ij}(t_k))
		=\|\bm{\Phi}_{ij}(t_k)\|^2\leq 1$, let $\rho_{ij}
		=
		\sqrt{1-
			\frac{\lambda_{\min}^2(\mathcal{S}_{ij})}
			{(1+\lambda_{\max}(\mathcal{S}_{ij}))^2}}$.
		Under Assumption~\ref{a1}, $0<\rho_{ij}<1$. 
		It follows by induction that $\|\tilde{\bm{\Theta}}_{ij}(t_k)\|
		\leq
		\rho_{ij}^{\,k-k_c}
		\|\tilde{\bm{\Theta}}_{ij}(t_{k_c})\|,
		k\geq k_c$.
		Since $H_{ij}$ is globally Lipschitz in its first argument,
		the initial relative pose estimation error
		$\tilde{\bm{\Theta}}^{*}_{ij}(t_k)$ also converges exponentially
		to zero. 
		The same conclusion holds for all $t\geq t_{k_c}$
		because the estimate is held constant between consecutive
		sampling instants.
	\end{proof}
	
	\blue{Based on the $\hat{\bm{q}}_{ij}^{\mathcal{O}_{i}}(t)$ and $\hat{\bm{Q}}_{\mathcal{O}_{i}}^{\mathcal{O}_{j}}(t)$, the real-time RL estimators $\hat{\bm{q}}_{ij}^{\Sigma_{i}}(t)$ and $\hat{\bm{Q}}_{\Sigma_{i}}^{\Sigma_{j}}(t)$ are designed to estimate the real-time relative pose $\bm{p}_{ij}^{\Sigma_{i}}$ and $\bm{R}_{\Sigma_{i}}^{\Sigma_{j}}(t)$ between robot $i$ and $j$:}
	\blue{\begin{align}
			\label{chat5_real_time_relative_estimator}
			\notag
			\hat{\bm{q}}_{ij}^{\Sigma_{i}}(t) &= \bm{R}^{\mathcal{O}_{i}}_{\Sigma_{i}}(t)^{T} \left( \hat{\bm{q}}^{\mathcal{O}_{i}}_{ij}(t) + \bm{z}^{\mathcal{O}_{i}}_{i}(t) - \hat{\bm{Q}}^{\mathcal{O}_{j}}_{\mathcal{O}_{i}}(t)^{T} \bm{z}^{\mathcal{O}_{j}}_{j}(t) \right),
			\\
			\hat{\bm{Q}}_{\Sigma_{i}}^{\Sigma_{j}}(t) &= \bm{R}^{\mathcal{O}_{j}}_{\Sigma_{j}}(t)^T
			\hat{\bm{Q}}^{\mathcal{O}_{j}}_{\mathcal{O}_{i}}(t) \bm{R}^{\mathcal{O}_{i}}_{\Sigma_{i}}(t).
	\end{align}}
	Then the RL between robots $i$ and $j$ can be achieved.
	Other relative measurements (e.g., ratio-of-distance or angle) can be integrated into the observation equation of $\bm{\Theta}_{ij}$ via odometer augmentation, yielding a unified method.
	\blue{\begin{remark}
			Existing studies indicate that relative measurements such as distance or bearing alone are generally insufficient to directly determine the complete relative pose between two neighboring robots. 
			To address this limitation, we incorporate odometric displacement measurements, which are readily available on robots. 
			An observation equation is then constructed to relate the heterogeneous measurements collected by any robot pair at different time instants to the unknown pose parameters. 
			Theorem \ref{t1} shows that, given sufficient recorded data, the proposed RL estimator can estimate the complete relative pose using only the relative measurements and odometry of that pair.
	\end{remark}}
	
	\section{Integrated cooperative localization}
	\label{CL}
	In this section, we present a distributed CL method based on the adaptive RL estimator and demonstrate that, under sufficient data collection, the CL estimator converges asymptotically.
	\subsection{Cooperative localization estimator}
	\blue{Design the following CL estimators $\hat{\bm{p}}^{\Sigma_i}_{i0}(t)$ and $\hat{\bm{R}}^{\Sigma_0}_{\Sigma_i}(t)$ to estimate the real-time relative pose $\bm{p}_{i0}^{\Sigma_{i}}(t)$ and $\bm{R}_{\Sigma_{i}}^{\Sigma_{0}}(t)$ between robot $i$ and the anchor $0$}
	\begin{align}
		\label{cooperative estimator}
		\notag
		\hat{\bm{p}}^{\Sigma_{i}}_{i0}(t) &=  \bm{\bm{R}}^{\mathcal{O}_{i}}_{\Sigma_{i}}(t)^{T} \left( \hat{\bm{\bm{R}}}_{\mathcal{O}_{i}}^{\mathcal{O}_{0}}(t)^{T} \left(\hat{\bm{p}}_{i}^{\mathcal{O}_{0}}(t) - \hat{\bm{z}}_{i}(t)\right) 
		+ \bm{z}^{\mathcal{O}_{i}}_{i}(t)\right),\\
		\hat{\bm{R}}_{\Sigma_{i}}^{\Sigma_{0}}(t) &= \hat{\bm{R}}_{i}(t)^{T} \hat{\bm{R}}_{\mathcal{O}_{i}}^{\mathcal{O}_{0}} \bm{R}_{\Sigma_{i}}^{\mathcal{O}_{i}},   
	\end{align}
	where $\bm{R}_{\Sigma_{i}}^{\mathcal{O}_{i}} $ and $\bm{z}^{\mathcal{O}_{i}}_{i}$ are available to robot $i$.
	\blue{$\hat{\bm{p}}_{i}^{\mathcal{O}_{0}}$ and $\hat{\bm{R}}_{\mathcal{O}_{i}}^{\mathcal{O}_{0}}$ are the distributed coupled estimators to estimate the unknown initial relative pose $\bm{p}_{i0}^{\mathcal{O}_{0}}(t_{0})$ and $\bm{R}_{\mathcal{O}_{i}}^{\mathcal{O}_{0}}$ between robot $i$ and the anchor.}
	\blue{$\hat{\bm{z}}_{i}$ and $\hat{\bm{R}}_{i}$ are distributed observers for robot $i$ to estimate the anchor's odometry information $\bm{z}_{0}^{\mathcal{O}_{0}}$ and $\bm{R}_{\Sigma_{0}}^{\mathcal{O}_{0}}$.}
	
	\subsubsection{Distributed coupled estimator design}
	For robot $i$, estimators $\hat{\bm{p}}_{i}^{\mathcal{O}_{0}}$ and $\hat{\bm{R}}_{\mathcal{O}_{i}}^{\mathcal{O}_{0}} = \bm{R}\left(\hat{C}_{\mathcal{O}_{i}}^{\mathcal{O}_{0}}, \hat{S}_{\mathcal{O}_{i}}^{\mathcal{O}_{0}} \right)$ are designed as, 
	\blue{\begin{align}
			\label{initial_pose_estimator}
			\notag
			\dot{\hat{\bm{p}}}_{i}^{\mathcal{O}_{0}} = &\sum_{j\in \mathcal{N}_{i}} \left( \hat{\bm{p}}_{j}^{\mathcal{O}_{0}} + \hat{\bm{R}}_{\mathcal{O}_{i}}^{\mathcal{O}_{0}} \hat{\bm{q}}^{\mathcal{O}_{i}}_{ij} - \hat{\bm{p}}_{i}^{\mathcal{O}_{0}} \right) + \mu_{i} \left( \hat{\bm{Q}}_{\mathcal{O}_{i}}^{\mathcal{O}_{0}} \hat{\bm{q}}^{\mathcal{O}_{i}}_{i0} - \hat{\bm{p}}_{i}^{\mathcal{O}_{0}} \right) \\
			\notag
			\dot{\hat{C}}_{\mathcal{O}_{i}}^{\mathcal{O}_{0}} = &\sum_{j\in \mathcal{N}_{i}} \left( \hat{c}_{\mathcal{O}_{i}}^{\mathcal{O}_{j}} \hat{C}_{\mathcal{O}_{j}}^{\mathcal{O}_{0}} - \hat{s}_{\mathcal{O}_{i}}^{\mathcal{O}_{j}} \hat{S}_{\mathcal{O}_{j}}^{\mathcal{O}_{0}} - \hat{C}_{\mathcal{O}_{i}}^{\mathcal{O}_{0}} \right)
			+ \mu_{i} \left( \hat{c}_{\mathcal{O}_{i}}^{\mathcal{O}_{0}} - \hat{C}_{\mathcal{O}_{i}}^{\mathcal{O}_{0}} \right) \\
			\dot{\hat{S}}_{\mathcal{O}_{i}}^{\mathcal{O}_{0}} = &\sum_{j\in \mathcal{N}_{i}} \left( \hat{s}_{\mathcal{O}_{i}}^{\mathcal{O}_{j}} \hat{C}_{\mathcal{O}_{j}}^{\mathcal{O}_{0}} + \hat{c}_{\mathcal{O}_{i}}^{\mathcal{O}_{j}} \hat{S}_{\mathcal{O}_{j}}^{\mathcal{O}_{0}} - \hat{S}_{\mathcal{O}_{i}}^{\mathcal{O}_{0}} \right)
			+ \mu_{i} \left( \hat{s}_{\mathcal{O}_{i}}^{\mathcal{O}_{0}} - \hat{S}_{\mathcal{O}_{i}}^{\mathcal{O}_{0}} \right),
	\end{align}}
	where $\hat{\bm{q}}^{\mathcal{O}_{i}}_{ij}$, $\hat{c}_{\mathcal{O}_{i}}^{\mathcal{O}_{j}}$ and $\hat{s}_{\mathcal{O}_{i}}^{\mathcal{O}_{j}}$ are defined in Section \ref{RL}. 
	
	\subsubsection{Distributed observer design}
	For robot $i$ the distributed observers $\hat{\bm{z}}_{i}$ and  $\hat{\bm{R}}_{i} = \bm{R}(\hat{c}_{i}, \hat{s}_{i})$ are designed as
	\begin{align}
		\label{distributed observer}
		&\dot{\hat{\bm{z}}}_{i}(t) = -k_{1} \bm{e}_{zi} - \frac{\bm{e}_{zi}}{\sqrt{\bm{e}_{zi}^T\bm{e}_{zi} + \delta}} v_{max}, \\ \notag
		&\dot{\hat{c}}_{i}(t) = -k_{1} e_{ci} - \frac{e_{ci}}{\sqrt{e_{ci}^2 + \delta}} w_{max},  \\ \notag
		&\dot{\hat{s}}_{i}(t) = -k_{1} e_{si} - \frac{e_{si}}{\sqrt{e_{si}^2 + \delta}} w_{max},
	\end{align}
	where $\delta=\beta e^{-\alpha t}$, $\beta, \alpha>0$.
	$\bm{e}_{zi} = \sum_{j\in \mathcal{N}_{i}} \left( \hat{\bm{z}}_{i} - \hat{\bm{z}}_{j}\right) + \mu_{i} \left( \hat{\bm{z}}_{i} - \bm{z}^{\mathcal{O}_{0}}_{0}\right)$, $e_{ci} = \sum_{j\in \mathcal{N}_{i}} \left( \hat{c}_{i} - \hat{c}_{j}\right) + \mu_{i} \left( \hat{c}_{i} - \cos \theta_{\Sigma_{0}}^{\mathcal{O}_{0}}\right) $ and $e_{si} = \sum_{j\in \mathcal{N}_{i}} \left( \hat{s}_{i} - \hat{s}_{j}\right) + \mu_{i} \left( \hat{s}_{i} - \sin \theta_{\Sigma_{0}}^{\mathcal{O}_{0}}\right) $ are the consensus errors. 
	Pseudocode of the proposed method is provided in Appendix for better illustration.
	\blue{\begin{remark}
			Most existing CL methods rely on multilateral geometric constraints.
			To determine its own pose relative to the anchor, each robot constructs
			geometric constraints from \textbf{its relative measurements to multiple neighbors} and those neighbors' poses relative to the anchor.
			This usually requires a sufficient number of neighbors in a suitable
			geometric configuration and thus imposes structural requirements on the
			measurement graph, such as multi-layer graphs~\cite{xie2023tac},
			Henneberg constructions~\cite{lv2024tii}, or other graph structures~\cite{xie2025tro,lcs2024bounded}.
			Under Assumption \ref{a1}, the RL estimator recovers the relative pose $\big(\hat{\bm{q}}_{ij}^{\Sigma_i},\hat{\bm{Q}}_{\Sigma_i}^{\Sigma_j}\big)$ from the two robots' odometry and relative measurements. 
			A robot can then determine its pose relative to the anchor by composing
			\textbf{its relative pose to a single neighbor} with the neighbor's pose relative to the anchor.
			This removes the need for each robot to construct geometric constraints from measurements involving multiple neighbors and relaxes the requirements on the measurement topology.
	\end{remark}}
	
	\subsection{Convergence analysis}  
	\blue{Redefine $\hat{\bm{R}}_{\Sigma_{i}}^{\Sigma_{0}} = \bm{R}\left( \hat{C}_{\Sigma_{i}}^{\Sigma_{0}}, \hat{S}_{\Sigma_{i}}^{\Sigma_{0}} \right)$ and define the closed-loop real-time estimation error as $\tilde{\bm{p}}^{\Sigma_{i}}_{i0} = \hat{\bm{p}}^{\Sigma_{i}}_{i0} - \bm{p}^{\Sigma_{i}}_{i0}$, 
		$\tilde{C}_{\Sigma_{i}}^{\Sigma_{0}} = \hat{C}_{\Sigma_{i}}^{\Sigma_{0}} - \cos \theta_{{\Sigma}_{i}}^{{\Sigma}_{0}}$ and $\tilde{S}_{\Sigma_{i}}^{\Sigma_{0}} = \hat{S}_{\Sigma_{i}}^{\Sigma_{0}} - \sin \theta_{{\Sigma}_{i}}^{{\Sigma}_{0}}$.
		Firstly, we establish the asymptotic convergence of the coupled estimators (i.e. $\hat{\bm{p}}_{i}^{\mathcal{O}_{0}}$ and $\hat{\bm{R}}_{\mathcal{O}_{i}}^{\mathcal{O}_{0}}$) and the distributed observers (i.e. $\hat{\bm{z}}_{i}$ and $\hat{\bm{R}}_{i}$).} 
	
	\subsubsection{Convergence of the distributed coupled estimator}
	In this section, we establish the asymptotic convergence of the distributed coupled estimators. 
	\blue{Define the estimation error of the distributed coupled observers as
		$\tilde{\bm{p}}_{i}^{\mathcal{O}_{0}}(t)
		= \hat{\bm{p}}_{i}^{\mathcal{O}_{0}}(t)
		-\bm{p}_{i0}^{\mathcal{O}_{0}}(t_{0})$,
		$\tilde{C}_{\mathcal{O}_{i}}^{\mathcal{O}_{0}}(t)
		= \hat{C}_{\mathcal{O}_{i}}^{\mathcal{O}_{0}}(t)
		-\cos\theta_{\mathcal{O}_{i}}^{\mathcal{O}_{0}}$
		and
		$\tilde{S}_{\mathcal{O}_{i}}^{\mathcal{O}_{0}}(t)
		= \hat{S}_{\mathcal{O}_{i}}^{\mathcal{O}_{0}}(t)
		-\sin\theta_{\mathcal{O}_{i}}^{\mathcal{O}_{0}}$.}
	To facilitate the convergence analysis, the following definition and lemma are introduced.
	
	\begin{definition}
		Define the initial relative yaw angle between robots $i$ and $j$ as $\theta_{\mathcal{O}_{i}}^{\mathcal{O}_{j}}$, take the attitude Laplacian matrix as 
		\begin{align}
			\notag
			\mathcal{L}_{\bar{\bm{R}}} = \begin{bmatrix} 
				\lvert \mathcal{N}_{1} \rvert \bm{I}_{2 \times 2} &\cdots &-a_{1i} \bar{\bm{R}}_{\mathcal{O}_{1}}^{\mathcal{O}_{i}} &\cdots &-a_{1N} \bar{\bm{R}}_{\mathcal{O}_{1}}^{\mathcal{O}_{N}} \\
				&\vdots & \ddots &\vdots & \ddots &\vdots \\
				-a_{i1} \bar{\bm{R}}_{\mathcal{O}_{i}}^{\mathcal{O}_{1}} & \ddots &\lvert \mathcal{N}_{i} \rvert \bm{I}_{2 \times 2} & \ddots &-a_{iN} \bar{\bm{R}}_{\mathcal{O}_{i}}^{\mathcal{O}_{N}} \\
				&\vdots & \ddots &\vdots & \ddots &\vdots \\
				-a_{N1} \bar{\bm{R}}_{\mathcal{O}_{N}}^{\mathcal{O}_{1}} & \ddots &-a_{Ni} \bar{\bm{R}}_{\mathcal{O}_{N}}^{\mathcal{O}_{i}} & \ddots &\lvert \mathcal{N}_{N} \rvert \bm{I}_{2 \times 2} \\
			\end{bmatrix},
		\end{align}
		where $\bar{\bm{R}}_{\mathcal{O}_{i}}^{\mathcal{O}_{j}}$ is defined as $\bar{\bm{R}}_{\mathcal{O}_{i}}^{\mathcal{O}_{j}}=\begin{bmatrix} \cos \theta_{\mathcal{O}_{i}}^{\mathcal{O}_{j}}, & -\sin \theta_{\mathcal{O}_{i}}^{\mathcal{O}_{j}} \\ \sin \theta_{\mathcal{O}_{i}}^{\mathcal{O}_{j}}, & \cos \theta_{\mathcal{O}_{i}}^{\mathcal{O}_{j}} \end{bmatrix}$.
	\end{definition}
	
	\begin{lemma}
		\label{lemma1}
		If the communication topology $\mathcal{G}_{c}=(\mathcal{V},\mathcal{E}_{c})$ is connected (i.e. Assumption \ref{top} holds) and at least one robot $i$ is informed, i.e. $\exists i, \mu_{i}=1$, the coupled information matrix $\left( \mathcal{L}_{\bar{\bm{R}}} + \mathcal{B} \otimes \bm{I}_{2 \times 2} \right)$ is positive definite.
	\end{lemma}
	\begin{proof}
		To prove $\left( \mathcal{L}_{\bar{\bm{R}}} + \mathcal{B} \otimes \bm{I}_{2 \times 2} \right)>0$, we firstly prove that $\forall \bm{x} \in \mathbb{R}^{2N}$, inequality $\bm{x}^{T} \left( \mathcal{L}_{\bar{\bm{R}}} + \mathcal{B} \otimes \bm{I}_{2 \times 2} \right) \bm{x} \geq 0$ holds, then we show that the equal sign only holds when $\bm{x} = \bm{0}$.
		Let $\bm{x} = [x_{11},x_{12},...,x_{N1},x_{N2}]^{T} = [\bm{x}_{1}^{T},...,\bm{x}_{N}^{T}]^{T}$, then $\bm{x}^{T} \left( \mathcal{L}_{\bar{\bm{R}}} + \mathcal{B} \otimes \bm{I}_{2 \times 2} \right) \bm{x}$ can be represented as
		\begin{align}
			\label{analysis}
			&\bm{x}^{T} \left( \mathcal{L}_{\bar{\bm{R}}} + \mathcal{B} \otimes \bm{I}_{2 \times 2} \right) \bm{x} \\
			\notag
			= &\sum_{i=1}^{N} \sum_{j\in \mathcal{N}_{i}} - \bm{x}_{i}^{T} \bar{\bm{R}}_{\mathcal{O}_{i}}^{\mathcal{O}_{j}} \bm{x}_{j} + \sum_{i=1}^{N} \lvert \mathcal{N}_{i} \rvert \bm{x}_{i}^{T} \bm{x}_{i} + \sum_{i=1}^{N} \mu_{i} \bm{x}_{i}^{T} \bm{x}_{i}\\
			\notag
			= & \sum_{\substack{1\leq i<j\leq N\\ a_{ij}=1}} \left( \bm{x}_{i} - \bar{\bm{R}}_{\mathcal{O}_{i}}^{\mathcal{O}_{j}}\bm{x}_{j}\right)^{T} \left( \bm{x}_{i} - \bar{\bm{R}}_{\mathcal{O}_{i}}^{\mathcal{O}_{j}}\bm{x}_{j}\right) + \sum_{i=1}^{N} \mu_{i} \bm{x}_{i}^{T} \bm{x}_{i}\geq 0.
		\end{align}
		Note that the second equation is due to the fact that $\bm{x}_{i}^{T} \bar{\bm{R}}_{\mathcal{O}_{i}}^{\mathcal{O}_{j}} \bm{x}_{j} = \left( \bm{x}_{i}^{T} \bar{\bm{R}}_{\mathcal{O}_{i}}^{\mathcal{O}_{j}} \bm{x}_{j} \right)^{T} = \bm{x}_{j}^{T} \bar{\bm{R}}_{\mathcal{O}_{j}}^{\mathcal{O}_{i}} \bm{x}_{i}$. Next, we analyze the situation when the equal sign holds.
		According to \eqref{analysis}, when equality holds,
		$\bm{x}_{i}=\bm{0}$ for every $i\in F$ with $\mu_{i}=1$, and
		$\bm{x}_{i}=\bar{\bm{R}}_{\mathcal{O}_{i}}^{\mathcal{O}_{j}}\bm{x}_{j}$
		for every $i\in F$ and $j\in\mathcal{N}_{i}$.
		Since $\mathcal{G}_{c}$ is connected, each connected component of
		$\mathcal{G}_{c}[F]$ contains at least one informed robot.
		By the invertibility of the rotation matrices, the edge relations
		then imply that $\bm{x}_{i}=\bm{0}$ for every robot in that component.
		Hence, $\bm{x}=\bm{0}$.
		Then the coupled information matrix $\left( \mathcal{L}_{\bar{\bm{R}}} + \mathcal{B} \otimes \bm{I}_{2 \times 2} \right)$ is positive definite.
	\end{proof}
	\blue{
		Define the orientation estimation errors as
		$\tilde{\bm{\chi}}
		= [\tilde{C}_{\mathcal{O}_{1}}^{\mathcal{O}_{0}},
		\tilde{S}_{\mathcal{O}_{1}}^{\mathcal{O}_{0}},\ldots,
		\tilde{C}_{\mathcal{O}_{N}}^{\mathcal{O}_{0}},
		\tilde{S}_{\mathcal{O}_{N}}^{\mathcal{O}_{0}}]^{T}$
		and position estimation error as
		$\tilde{\bm{p}}^{\mathcal{O}_{0}}
		= [(\tilde{\bm{p}}_{1}^{\mathcal{O}_{0}})^{T},\ldots,
		(\tilde{\bm{p}}_{N}^{\mathcal{O}_{0}})^{T}]^{T}$.
		The convergence of the CL estimator
		(\ref{initial_pose_estimator}) can be described
		in the following Theorem \ref{t2}.}
	
	\begin{theorem}
		\label{t2}
		Consider the coupled distributed estimator (\ref{initial_pose_estimator}), suppose that Assumption~\ref{top} holds and that Assumption~\ref{a1} is satisfied for every neighboring pair in $\mathcal G_c$, then the estimation error $\tilde{\bm{p}}^{\mathcal{O}_{0}}$ and $\tilde{\bm{\chi}}$ converge to $\bm{0}$ asymptotically.
	\end{theorem}
	
	\proof{Firstly, we show the asymptotic convergence of the orientation observer
		$\hat{C}_{\mathcal{O}_{i}}^{\mathcal{O}_{0}}(t)$ and
		$\hat{S}_{\mathcal{O}_{i}}^{\mathcal{O}_{0}}(t)$.
		Considering the distributed coupled observer
		(\ref{initial_pose_estimator}), the derivative of
		$\tilde{C}_{\mathcal{O}_{i}}^{\mathcal{O}_{0}}(t)$ and
		$\tilde{S}_{\mathcal{O}_{i}}^{\mathcal{O}_{0}}(t)$
		can be represented as
		\blue{\begin{align}
				\notag
				\dot{\tilde{C}}_{\mathcal{O}_{i}}^{\mathcal{O}_{0}}
				={}&
				\sum_{j\in\mathcal{N}_{i}}
				\left(
				\cos\theta_{\mathcal{O}_{i}}^{\mathcal{O}_{j}}
				\tilde{C}_{\mathcal{O}_{j}}^{\mathcal{O}_{0}}
				-\sin\theta_{\mathcal{O}_{i}}^{\mathcal{O}_{j}}
				\tilde{S}_{\mathcal{O}_{j}}^{\mathcal{O}_{0}}
				-\tilde{C}_{\mathcal{O}_{i}}^{\mathcal{O}_{0}}
				\right)
				-\mu_i\tilde{C}_{\mathcal{O}_{i}}^{\mathcal{O}_{0}}
				\\
				&+
				\sum_{j\in\mathcal{N}_{i}}
				\left(
				\tilde{c}_{\mathcal{O}_{i}}^{\mathcal{O}_{j}}
				\hat{C}_{\mathcal{O}_{j}}^{\mathcal{O}_{0}}
				-\tilde{s}_{\mathcal{O}_{i}}^{\mathcal{O}_{j}}
				\hat{S}_{\mathcal{O}_{j}}^{\mathcal{O}_{0}}
				\right)
				+\mu_i\tilde{c}_{\mathcal{O}_{i}}^{\mathcal{O}_{0}},
				\\
				\notag
				\dot{\tilde{S}}_{\mathcal{O}_{i}}^{\mathcal{O}_{0}}
				={}&
				\sum_{j\in\mathcal{N}_{i}}
				\left(
				\sin\theta_{\mathcal{O}_{i}}^{\mathcal{O}_{j}}
				\tilde{C}_{\mathcal{O}_{j}}^{\mathcal{O}_{0}}
				+\cos\theta_{\mathcal{O}_{i}}^{\mathcal{O}_{j}}
				\tilde{S}_{\mathcal{O}_{j}}^{\mathcal{O}_{0}}
				-\tilde{S}_{\mathcal{O}_{i}}^{\mathcal{O}_{0}}
				\right)
				-\mu_i\tilde{S}_{\mathcal{O}_{i}}^{\mathcal{O}_{0}}
				\\
				&+
				\sum_{j\in\mathcal{N}_{i}}
				\left(
				\tilde{s}_{\mathcal{O}_{i}}^{\mathcal{O}_{j}}
				\hat{C}_{\mathcal{O}_{j}}^{\mathcal{O}_{0}}
				+\tilde{c}_{\mathcal{O}_{i}}^{\mathcal{O}_{j}}
				\hat{S}_{\mathcal{O}_{j}}^{\mathcal{O}_{0}}
				\right)
				+\mu_i\tilde{s}_{\mathcal{O}_{i}}^{\mathcal{O}_{0}}.
		\end{align}}
		
		Take the derivative of $\tilde{\bm{\chi}}$ and one can check that 
		\begin{align}
			\label{initial_orientation_error}
			\dot{\tilde{\bm{\chi}}}
			={}&
			-\left(
			\mathcal{L}_{\bar{\bm{R}}}
			+\mathcal{B}\otimes\bm{I}_{2\times2}
			\right)\tilde{\bm{\chi}}
			+\bm{g}_a
			\left(t,\tilde{\bm{\chi}},\bm{\Xi}_a\right).
		\end{align}
		The disturbance vector is
		$\bm{g}_a=[\bm{g}_{a1}^{T},\ldots,\bm{g}_{aN}^{T}]^{T}$, where
		\begin{align}
			\notag
			\bm{g}_{ai}
			={}&
			\sum_{j\in\mathcal{N}_i}
			\begin{bmatrix}
				\tilde{c}_{\mathcal{O}_i}^{\mathcal{O}_j}
				&-\tilde{s}_{\mathcal{O}_i}^{\mathcal{O}_j}\\
				\tilde{s}_{\mathcal{O}_i}^{\mathcal{O}_j}
				&\tilde{c}_{\mathcal{O}_i}^{\mathcal{O}_j}
			\end{bmatrix}
			\begin{bmatrix}
				\hat{C}_{\mathcal{O}_j}^{\mathcal{O}_0}\\
				\hat{S}_{\mathcal{O}_j}^{\mathcal{O}_0}
			\end{bmatrix}
			+\mu_i
			\begin{bmatrix}
				\tilde{c}_{\mathcal{O}_i}^{\mathcal{O}_0}\\
				\tilde{s}_{\mathcal{O}_i}^{\mathcal{O}_0}
			\end{bmatrix}.
		\end{align}
		
		where $\bm{\Xi}_{a}\in \mathbb{R}^{2N}$ is defined as $\bm{\Xi}_a
		={}
		\Big[
		\sum_{j\in\mathcal{N}_1}
		\left|\tilde{c}_{\mathcal{O}_1}^{\mathcal{O}_j}\right|
		+\mu_1
		\left|\tilde{c}_{\mathcal{O}_1}^{\mathcal{O}_0}\right|, 
		\sum_{j\in\mathcal{N}_1}
		\left|\tilde{s}_{\mathcal{O}_1}^{\mathcal{O}_j}\right|
		+\mu_1
		\left|\tilde{s}_{\mathcal{O}_1}^{\mathcal{O}_0}\right|,
		\ldots,
		\sum_{j\in\mathcal{N}_N}
		\left|\tilde{c}_{\mathcal{O}_N}^{\mathcal{O}_j}\right|
		+\mu_N
		\left|\tilde{c}_{\mathcal{O}_N}^{\mathcal{O}_0}\right|, 
		\sum_{j\in\mathcal{N}_N}
		\left|\tilde{s}_{\mathcal{O}_N}^{\mathcal{O}_j}\right|
		+\mu_N
		\left|\tilde{s}_{\mathcal{O}_N}^{\mathcal{O}_0}\right|
		\Big]^{T} $.
		The disturbance term satisfies
		\begin{align}\left\|
			\bm{g}_a
			\left(
			t,\tilde{\bm{\chi}},\bm{\Xi}_a
			\right)
			\right\|
			\leq{}
			\sqrt{2}\left\|\bm{\Xi}_a(t)\right\|
			\left\|\tilde{\bm{\chi}}(t)\right\|
			+
			\sqrt{2}\left\|\bm{\Xi}_a(t)\right\|
		\end{align}
		By Lemma \ref{lemma1}, let
		$\lambda_a=\lambda_{\min}(
		\mathcal L_{\bar{\bm R}}+\mathcal B\otimes\bm I_{2\times2})>0$.
		Consider $V_a=\frac12\|\tilde{\bm\chi}\|^2$.
		Between consecutive sampling instants, it has
		\begin{align}
			\notag
			\dot{V}_a
			\leq{}&
			-\left(\lambda_a-\sqrt{2}\|\bm{\Xi}_a\|\right)
			\|\tilde{\bm{\chi}}\|^2
			+\sqrt{2}\|\bm{\Xi}_a\|\|\tilde{\bm{\chi}}\|.
		\end{align}
		By Theorem \ref{t1}, there exist $T$, $C_a>0$, and $\gamma_a>0$
		such that, for all $t\geq T$,
		$\|\bm\Xi_a(t)\|\leq C_a e^{-\gamma_a(t-T)}$
		and $\sqrt2\|\bm\Xi_a(t)\|\leq\lambda_a/2$.
		Using Young's inequality yields
		\begin{align}
			\notag
			\dot{V}_a
			\leq{}&
			-\frac{\lambda_a}{2}V_a
			+\frac{2C_a^2}{\lambda_a}e^{-2\gamma_a(t-T)}.
		\end{align}
		Since $V_a$ is continuous at the sampling instants,
		the comparison principle implies that, for any
		$0<\eta_a<\min\{\lambda_a/4,\gamma_a\}$,
		there exists $M_a>0$ such that $\|\tilde{\bm{\chi}}(t)\|
		\leq M_a e^{-\eta_a(t-T)}, t\geq T$.
		Then the orientation estimation error converges
		exponentially.
		
		Then we show the convergence of position estimation observer $\hat{\bm{p}}_i^{\mathcal{O}_0}(t)$. 
		Considering the distributed coupled observer (\ref{initial_pose_estimator}), the derivative of $\tilde{\bm{p}}_i^{\mathcal{O}_0}(t)$ can be represented as 
		\begin{align}
			\label{initial_position_error}
			\dot{\tilde{\bm{p}}}_i^{\mathcal{O}_0}(t)
			=&
			\sum_{j\in\mathcal{N}_i}
			\left(
			\tilde{\bm{p}}_j^{\mathcal{O}_0}
			-\tilde{\bm{p}}_i^{\mathcal{O}_0}
			\right)
			-\mu_i\tilde{\bm{p}}_i^{\mathcal{O}_0}
			+\bm{\Xi}_{pi}.
		\end{align}
		Where $\bm{\Xi}_{pi} = \sum_{j\in\mathcal{N}_i}
		\left(
		\tilde{\bm{R}}_{\mathcal{O}_i}^{\mathcal{O}_0}
		\hat{\bm{q}}_{ij}^{\mathcal{O}_i}
		+\bm{R}_{\mathcal{O}_i}^{\mathcal{O}_0}
		\tilde{\bm{q}}_{ij}^{\mathcal{O}_i}
		\right) + \mu_i
		\left(
		\tilde{\bm{Q}}_{\mathcal{O}_i}^{\mathcal{O}_0}
		\hat{\bm{q}}_{i0}^{\mathcal{O}_i}
		+\bm{R}_{\mathcal{O}_i}^{\mathcal{O}_0}
		\tilde{\bm{q}}_{i0}^{\mathcal{O}_i}
		\right)$. 
		$\tilde{\bm{Q}}_{\mathcal{O}_{i}}^{\mathcal{O}_{j}}(t) = \hat{\bm{Q}}_{\mathcal{O}_{i}}^{\mathcal{O}_{j}}(t) - \bm{R}_{\mathcal{O}_{i}}^{\mathcal{O}_{j}}$ is the initial relative orientation estimation error matrix between robots $i$ and $j$ obtained from the pairwise RL estimator.
		$\tilde{\bm{R}}_{\mathcal{O}_{i}}^{\mathcal{O}_{0}}(t)
		=
		\hat{\bm{R}}_{\mathcal{O}_{i}}^{\mathcal{O}_{0}}(t)
		-
		\bm{R}_{\mathcal{O}_{i}}^{\mathcal{O}_{0}}$
		denotes the initial orientation estimation error matrix
		of robot $i$ relative to the anchor, obtained from the
		distributed coupled estimator.
		Then the derivative of $\tilde{\bm{p}}^{\mathcal{O}_0}$ can be calculated as
		\begin{align}
			\notag
			\dot{\tilde{\bm{p}}}^{\mathcal{O}_0}
			=
			-\left(
			\left(\mathcal{L}+\mathcal{B}\right)
			\otimes\bm{I}_{3\times3}
			\right)\tilde{\bm{p}}^{\mathcal{O}_0}
			+\bm{\Xi}_p,
		\end{align}
		where $\bm{\Xi}_p$ is defined as
		$\bm{\Xi}_p=[\bm{\Xi}_{p1}^{T},\ldots,\bm{\Xi}_{pN}^{T}]^{T}$.
		According to Lemma \ref{connected}, system
		$\dot{\tilde{\bm{p}}}^{\mathcal{O}_0}
		=-((\mathcal{L}+\mathcal{B})\otimes\bm{I}_{3\times3})
		\tilde{\bm{p}}^{\mathcal{O}_0}$
		is uniformly exponentially stable.
		By Theorem \ref{t1} and the exponential convergence of the
		orientation estimation error established above,
		$\bm{\Xi}_p$ converges to zero exponentially, since the initial
		relative position estimates are bounded.
		Then the exponentially stable nominal position error
		system driven by $\bm{\Xi}_p$ yields
		$\tilde{\bm{p}}^{\mathcal{O}_0}\to\bm{0}$.
	}
	
	\subsubsection{Convergence of the CL estimator}
	In this section, we establish the convergence of the real-time CL estimators. 
	\blue{Define the observer errors as $\tilde{\bm{z}}_{i}(t)$, $\tilde{c}_{i}(t)$ and $\tilde{s}_{i}(t)$ as $\tilde{\bm{z}}_{i} = \hat{\bm{z}}_{i} - \bm{z}^{\mathcal{O}_{0}}_{0}(t)$, $\tilde{c}_{i}(t) = \hat{c}_{i}(t) - \cos \theta_{\Sigma_{0}}^{\mathcal{O}_{0}}(t)$ and $\tilde{s}_{i}(t) = \hat{s}_{i}(t) - \sin \theta_{\Sigma_{0}}^{\mathcal{O}_{0}}(t)$.}
	The following theorem establishes the convergence of both the observer errors (i.e. $\tilde{\bm{z}}_{i}(t)$, $\tilde{c}_{i}(t)$ and $\tilde{s}_{i}(t)$) and real-time CL estimation errors (i.e. $\tilde{\bm{p}}^{\Sigma_{i}}_{i0}$, $\tilde{C}_{\Sigma_i}^{\Sigma_0}$ and $\tilde{S}_{\Sigma_i}^{\Sigma_0}$).
	
	\begin{theorem}
		Consider the CL estimator (\ref{cooperative estimator}) composed of the distributed coupled estimator (\ref{initial_pose_estimator}) and distributed observer (\ref{distributed observer}). 
		Suppose that Assumption~\ref{top} holds and that Assumption~\ref{a1} is satisfied for every neighboring pair in $\mathcal G_c$.
		Then the CL estimation errors $\lVert \tilde{\bm{p}}^{\Sigma_{i}}_{i0} \rVert$, $\tilde{C}_{\Sigma_i}^{\Sigma_0}$ and $\tilde{S}_{\Sigma_i}^{\Sigma_0}$ converge  asymptotically.
	\end{theorem}
	
	\proof{
		By Theorem \ref{t2}, $\tilde{\bm{p}}_i^{\mathcal{O}_0}\to\bm{0}$.
		Moreover, its proof establishes the exponential convergence of
		$\tilde{\bm R}_{\mathcal O_i}^{\mathcal O_0}$.
		The distributed observer (\ref{distributed observer}) ensures
		$\tilde{\bm z}_i\to\bm0$, $\tilde c_i\to0$, and $\tilde s_i\to0$ \cite{ze2023tie}.
		
		From (\ref{cooperative estimator}), the position estimation error
		satisfies
		\begin{align}
			\label{real_time_position_error}
			\|\tilde{\bm{p}}^{\Sigma_{i}}_{i0}\|
			\leq{}&
			\left\|
			\hat{\bm{R}}_{\mathcal{O}_i}^{\mathcal{O}_0}
			\right\|
			\left(
			\|\tilde{\bm{p}}_i^{\mathcal{O}_0}\|
			+\|\tilde{\bm{z}}_i\|
			\right)
			\\
			\notag
			&+
			\left\|
			\tilde{\bm{R}}_{\mathcal{O}_i}^{\mathcal{O}_0}
			\right\|
			\left(
			\|\bm{p}_{i0}^{\mathcal{O}_0}(t_0)\|
			+\|\bm{z}_0^{\mathcal{O}_0}(t)\|
			\right).
		\end{align}
		The first term tends to zero because
		$\hat{\bm R}_{\mathcal O_i}^{\mathcal O_0}$ is bounded.
		Since $\|\bm z_0^{\mathcal O_0}(t)\|\leq v_{\mathrm{max}}(t-t_0)$
		and $\tilde{\bm R}_{\mathcal O_i}^{\mathcal O_0}$ decays exponentially,
		the second term also tends to zero.
		Hence, $\tilde{\bm p}_{i0}^{\Sigma_{i}}\to\bm0$.
		
		Finally, both
		$\hat{\bm R}_{\mathcal O_i}^{\mathcal O_0}
		-\bm R_{\mathcal O_i}^{\mathcal O_0}$
		and
		$\hat{\bm R}_{i}
		-\bm R_{\Sigma_0}^{\mathcal O_0}$
		tend to zero.
		The boundedness of these matrices and the orientation reconstruction
		in (\ref{cooperative estimator}) therefore imply
		$\hat{\bm R}_{\Sigma_i}^{\Sigma_0}
		-\bm R_{\Sigma_i}^{\Sigma_0}\to\bm0$.
		Consequently, both $\tilde{C}_{\Sigma_i}^{\Sigma_0}$
		and $\tilde{S}_{\Sigma_i}^{\Sigma_0}$
		converge to zero, completing the proof.
	}
	
	\section{Simulation Results}
	\subsection{Application in formation control}
	\blue{The proposed CL method is applied to formation tracking under model \eqref{robot}.
		The anchor maintains its initial orientation in the task, with $w_0=0$ and $\bm R_{\Sigma_0}^{\mathcal O_0}=\bm I_3$.
		Let $\bm\rho_i^{\mathcal O_0}$ denote the constant desired position of follower $i$ relative to the anchor, expressed in $\mathcal O_0$.
		The control procedure consists of two stages.
		Taking $t_0=0$, all robots first collect localization data using
		\begin{align}
			\label{formation_collection}
			\bm{v}_i^{\mathcal{O}_i}
			=
			\begin{bmatrix}
				r_{xi}\nu_{xi}\sin(\nu_{xi}t)\\
				r_{yi}\nu_{yi}\cos(\nu_{yi}t)\\
				r_{zi}\nu_{zi}\cos(\nu_{zi}t)
			\end{bmatrix},
			\quad w_i=0,
		\end{align}
		where $r_{xi},r_{yi},r_{zi}$ and
		$\nu_{xi},\nu_{yi},\nu_{zi}$ are the motion amplitudes and
		frequencies, respectively.
		These parameters are selected so that the recorded data satisfy Assumption~\ref{a1}.}
	
	\blue{After sufficient data collection, the anchor continues the prescribed translational motion, while the followers perform formation tracking.
		Each follower estimates the anchor velocity
		$\bm v_0^{\mathcal O_0}$ through $\dot{\hat{\bm{v}}}_i =
		-k_1\bm{e}_{vi}
		-a_{\max}
		\frac{\bm{e}_{vi}}
		{\sqrt{\bm{e}_{vi}^{T}\bm{e}_{vi}+\delta}}$, where $\bm{e}_{vi}
		=
		\sum_{j\in\mathcal{N}_i}a_{ij}
		\left(\hat{\bm{v}}_i-\hat{\bm{v}}_j\right)
		+\mu_i
		\left(\hat{\bm{v}}_i-\bm{v}_0^{\mathcal{O}_0}\right)$, $k_1>0$,  $\delta=\beta e^{-\alpha t}$ with $\alpha,\beta>0$,
		and $a_{\max}$ is a known bound on the anchor acceleration.
		Only informed robots receive the anchor velocity.
		Define the estimated formation error as $\hat{\bm{\epsilon}}_{pi}
		=
		\hat{\bm{p}}_{i0}^{\Sigma_i}
		-
		\left(\hat{\bm{R}}_{\Sigma_i}^{\Sigma_0}\right)^T
		\bm{\rho}_i^{\mathcal{O}_0}$, $\hat{\epsilon}_{ci}
		=
		1-\hat{C}_{\Sigma_i}^{\Sigma_0}$ and $\hat{\epsilon}_{si}
		=
		\hat{S}_{\Sigma_i}^{\Sigma_0}$.
		The sine and cosine estimates are obtained directly
		from $\hat{\bm R}_{\Sigma_i}^{\Sigma_0}$.
		The true formation errors $\bm\epsilon_{pi}$, $\epsilon_{ci}$, and $\epsilon_{si}$ are defined analogously using the true relative pose.
		The follower's control inputs are 
		\begin{align}
			\label{formation_control}
			\bm{v}_i^{\mathcal{O}_i}
			&=
			\bm{R}_{\Sigma_i}^{\mathcal{O}_i}
			\left[
			\left(\hat{\bm{R}}_{\Sigma_i}^{\Sigma_0}\right)^T
			\hat{\bm{v}}_i
			-k_2\hat{\bm{\epsilon}}_{pi}
			\right], 
			w_i
			=-k_3\hat{\epsilon}_{si},
		\end{align}
		where $k_2,k_3>0$. 
		The control laws \eqref{formation_collection} and \eqref{formation_control}, together with the CL estimator \eqref{cooperative estimator}, are used to demonstrate cooperative formation control of a robotic swarm under heterogeneous measurement conditions.}
	
	\begin{figure}[!t]\centering
		\includegraphics[scale=1]{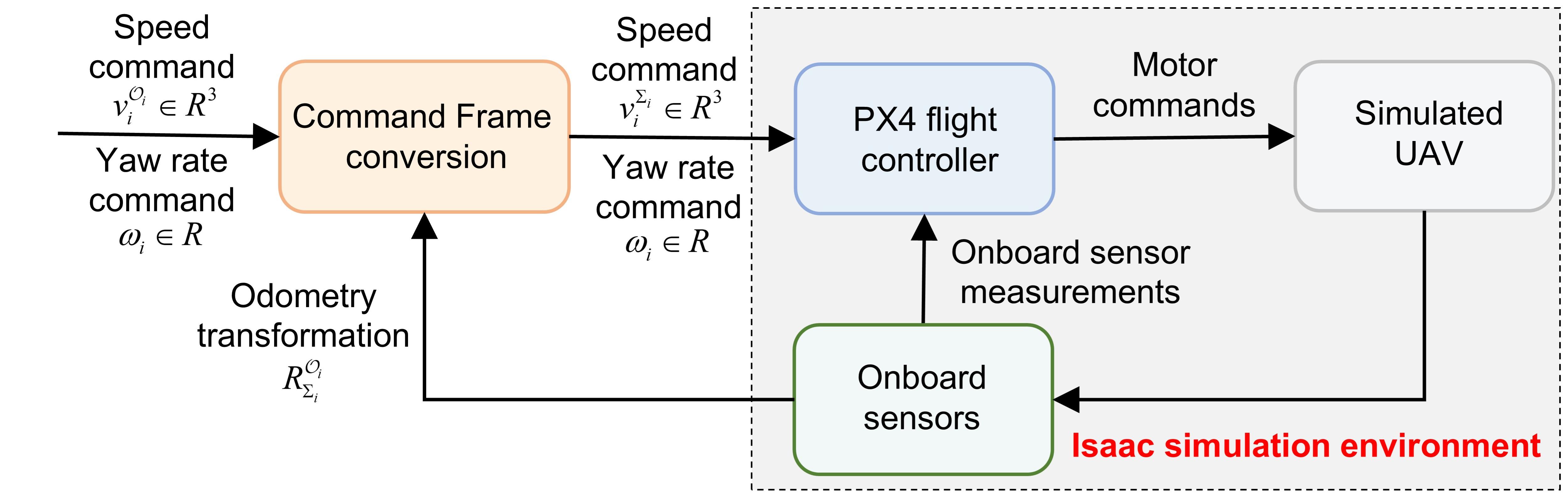}
		\caption{PX4 control architecture for each simulated UAV.}\centering
		\label{uav_controller}
		\vspace{-0.5cm}
	\end{figure}
	
	\subsection{Physical simulation results}
	\blue{We simulate six UAVs in the NVIDIA Isaac-sim environment \cite{jacinto2024pegasus}. 
		Each UAV has access to its own odometry, and a subset of the UAVs also carries sensors for relative measurements. Fig. \ref{topology} shows the swarm's heterogeneous measurement topology and communication topology.
		Each UAV uses its odometry to transform the velocity command in \eqref{robot} from the odometry frame to the body frame. 
		The PX4 controller accepts the resulting velocity and yaw rate commands  and generates motor commands for the simulated UAV, as shown in Fig. \ref{uav_controller}.}
	
	\blue{Among the six simulated UAVs, UAV $0$ serves as the anchor. 
		The simulation last for $120$s. 
		During the first $40$~s, all UAVs followed \eqref{formation_collection}. Each neighboring pair recorded one normalized regression sample every $0.5$s, yielding $\varsigma=80$ samples per pair.
		All the parameters used in the CL estimators are $\Delta t=0.02$s, $\kappa=10^{-4}$, $k_1=3$, $\alpha = 0.1$ and $\beta = 0.02$.
		At $t=40$~s, the followers switched to the formation control law \eqref{formation_control}. 
		The anchor continued along its prescribed trajectory with $w_0=0$ throughout the simulation.
		The parameters adopted in the control method are $v_{\max}=3$~m/s, $w_{\max}=1.2$~rad/s, $a_{\max}=2$~m/s$^2$. 
		The control gains are $k_2=0.6$ and $k_3=0.8$. 
		In \eqref{formation_collection}, we set $r_{yi}=r_{xi}$ and
		$\nu_{yi}=\nu_{xi}$ for each UAV. 
		In UAV index order, the amplitude pairs $(r_{xi},r_{zi})$ are $(0.4,0.1)$, $(0.6,0.7)$, $(0.3,0.3)$, $(0.5,0.9)$, $(0.4,0.3)$ and
		$(0.4,0.8)$m. 
		The corresponding pairs $(\nu_{xi},\nu_{zi})$ are $(0.2,0.2)$, $(0.3,0.5)$, $(-0.9,0.7)$, $(0.5,0.2)$, $(-0.2,0.7)$ and $(0.7,0.5)$rad/s.
		The desired formation offsets $\bm{\rho}_i^{\mathcal{O}_0}$
		for followers $1$ through $5$ are $(-2,0,0)$, $(-2,-2,0.4)$,
		$(0,-2,0)$, $(2,0,0.2)$, and $(2,-2,-0.4)$~m, respectively.}
	
	\begin{figure}[!t]\centering
		\includegraphics[scale=1.1]{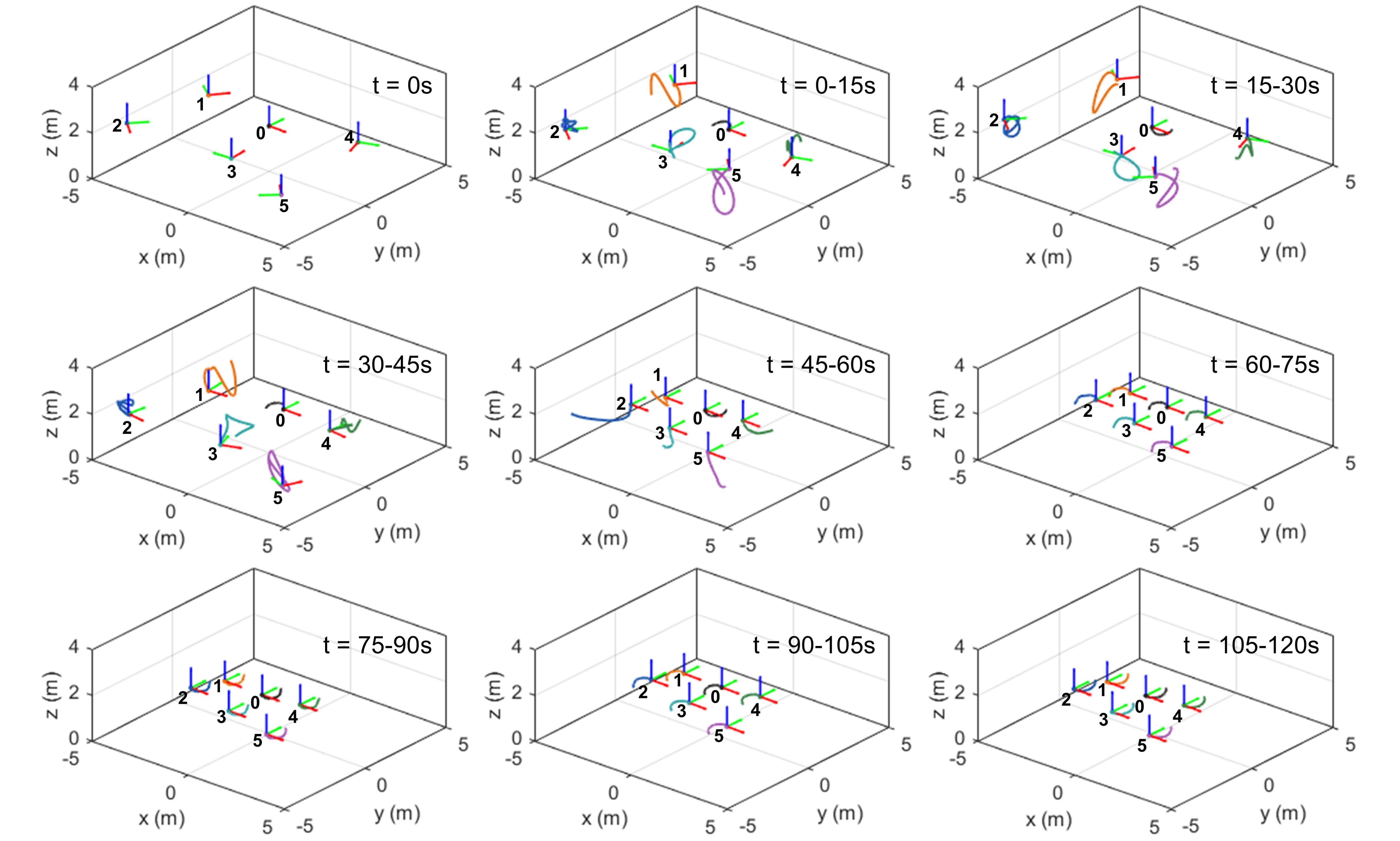}
		\caption{Snapshots of the UAV swarm during formation tracking}\centering
		\label{snapshot}
		\vspace{-0.5cm}
	\end{figure}
	
	\begin{figure}[!t]\centering
		\includegraphics[scale=1.1]{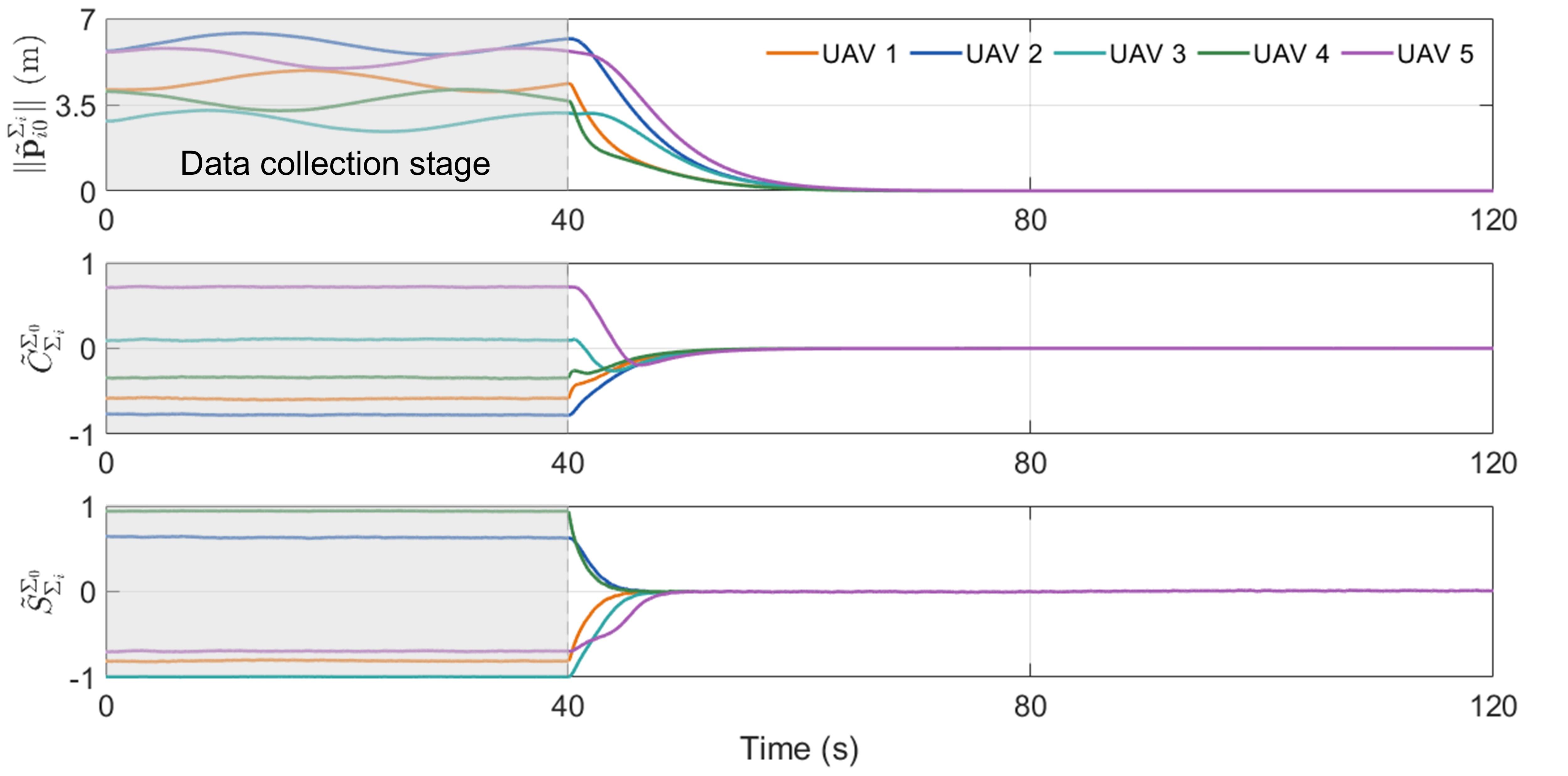}
		\caption{Cooperative localization estimation error of each UAV $i$. }\centering
		\label{estimation error}
		\vspace{-0.5cm}
	\end{figure}
	
	\blue{Fig. \ref{snapshot} shows the UAV trajectories during the first 40~s of data collection and the subsequent formation tracking phase.
		During the first 40~s, each neighboring pair collected $\varsigma=80$ regression samples. At $t=40$~s, the CL estimator and the formation tracking controller are activated.
		The CL estimation errors $\lVert \tilde{\bm{p}}^{\Sigma_{i}}_{i0} \rVert$, $\tilde{C}_{\Sigma_{i}}^{\Sigma_{0}}$ and $\tilde{S}_{\Sigma_{i}}^{\Sigma_{0}}$ are shown in Fig.\ref{estimation error}. 
		The formation tracking errors $\lVert \bm{\epsilon}_{pi} \rVert$, $\epsilon_{ci}$ and $\epsilon_{si}$ are shown in Fig.\ref{formation error}.}
	
	\section{Conclusion}
	In this paper, we have proposed a unified CL algorithm for a heterogeneous measurement swarm.
	With the proposed adaptive RL estimator, the RL task is achieved for each pair of neighboring robots in the measurement topology.
	The proposed pose coupling estimator realized the CL task under a weakly connected measurement topology, which is also the least restrictive measurement topology condition in existing results. 
	The effectiveness of the method is verified through its application to the formation control task of six UAVs.
	
	\begin{figure}[!t]\centering
		\includegraphics[scale=1.1]{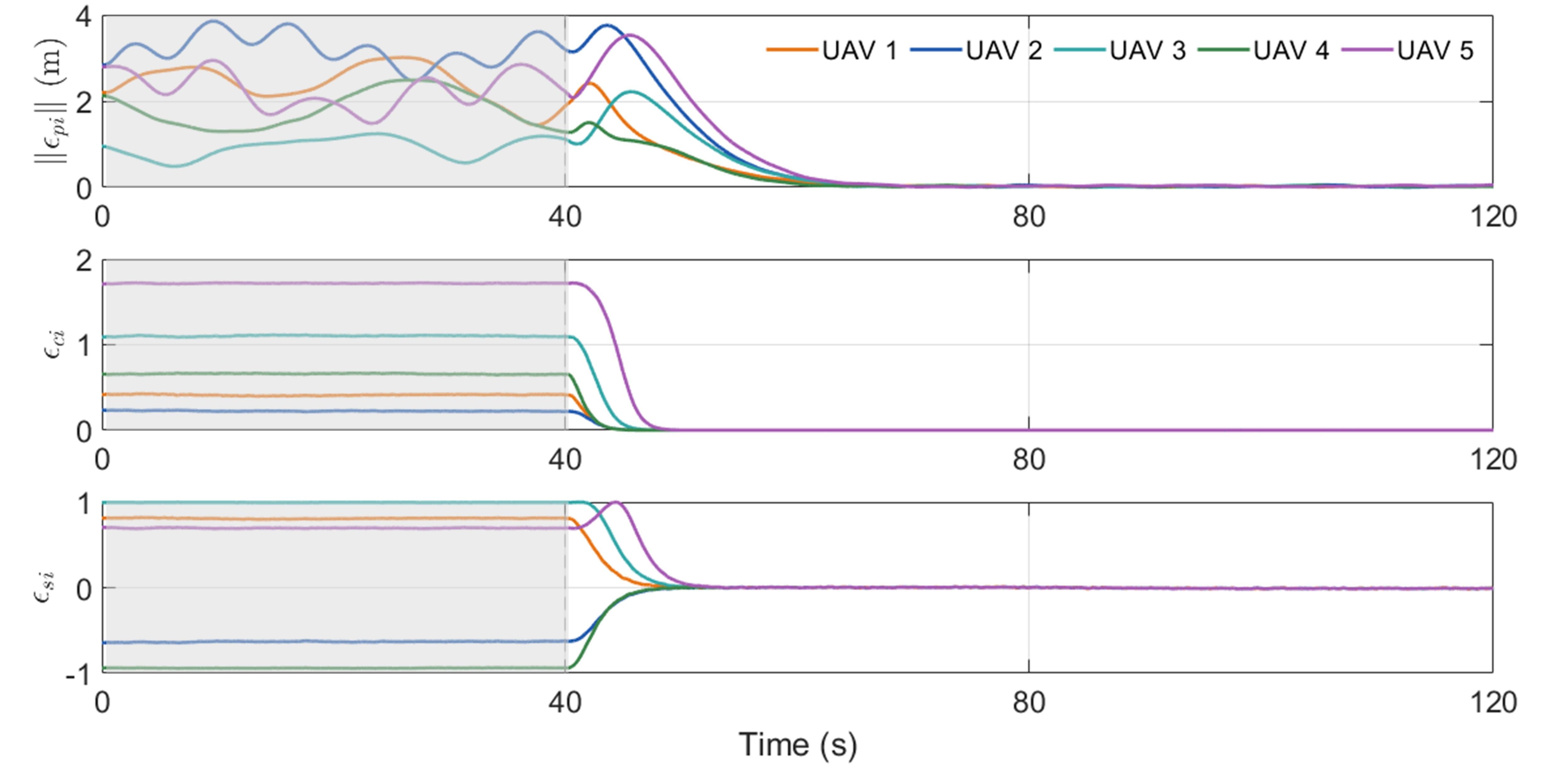}
		\caption{Formation tracking error of each UAV $i$. }\centering
		\label{formation error}
		\vspace{-0.5cm}
	\end{figure}
	\appendices

\section{Notation list of the proposed method}
\label{Notation}
For ease of reference, Table~\ref{Table1} summarizes the main symbols used in this paper.

\section{Pseudocode of the proposed method}
\label{Pseudocode}
\blue{Algorithm~\ref{RLCL} summarizes the localization procedure at robot~$i$.
	For each neighbor $j$, robot~$i$ constructs a scalar regression pair
	from the available relative measurements and the odometry of both robots.
	Once $\varsigma$ regression pairs have been recorded, robot~$i$ updates
	the parameter estimate $\hat{\bm{\Theta}}_{ij}$ at each sampling instant
	using the current and recorded data. 
	From this estimate and the initial relative measurements, it recovers the initial relative pose of the pair.}
	
\blue{The distributed coupled estimator combines these pairwise estimates with
	information from neighboring robots to obtain $\hat{\bm{p}}_i^{\mathcal O_0}$ and
	$\hat{\bm{R}}_{\mathcal O_i}^{\mathcal O_0}$, which estimate the initial
	position and orientation of robot~$i$ relative to the anchor $0$.
	In parallel, the distributed observer provides local estimates
	$\hat{\bm{z}}_i$ and $\hat{\bm{R}}_i$ of the anchor's displacement and orientation in $\mathcal O_0$.
	Finally, robot~$i$ combines its initial pose estimate, the estimated anchor motion, and its own odometry to obtain its real-time pose estimates $\hat{\bm{p}}_{i0}^{\Sigma_i}$ and $\hat{\bm{R}}_{\Sigma_i}^{\Sigma_0}$ relative to the anchor.}

\section{Discussion on the measurement noise}
\label{Noise}
Relative measurements from UWB devices or cameras are subject to errors.
We assume that these measurement errors are uniformly bounded
over the time interval considered.
In the proposed framework, neighboring robots estimate their relative
pose using the available pairwise measurements and their odometry.
Each non-anchor robot then uses these RL estimates and information
received from its neighbors to estimate its own pose relative to
the anchor. 
Measurement errors in the RL stage can therefore
propagate to the CL estimates.

We first consider the effect of measurement noise on pairwise RL.
Let $y_{ij}^{\mathrm{m}}(t_k)$ and $\bm{\Phi}_{ij}^{\mathrm{m}}(t_k)$
denote the regression output and regressor computed from the noisy
measurements, respectively. The resulting regression residual is
\begin{equation}
	\varepsilon_{ij}(t_k)
	= y_{ij}^{\mathrm{m}}(t_k)
	- \bm{\Phi}_{ij}^{\mathrm{m}}(t_k)^{T}\bm{\Theta}_{ij}(t_0).
\end{equation}
Since the regression residual before normalization is bounded on the
finite interval considered and the denominator is at least $\kappa>0$,
the normalized residuals are bounded for all current and recorded samples.
Under Assumption~\ref{a1}, it then follows from \eqref{relative estimator}
that the parameter estimation error $\tilde{\bm{\Theta}}_{ij}$ remains bounded.
Since the initial relative measurement errors are also bounded, \eqref{rl_estimator} yields bounded errors in the estimated initial relative position and
orientation of each pair of neighboring robots.

\begin{table}[!t]
	\caption{Main symbols used in this paper}
	\label{Table1}
	\centering
	\footnotesize
	\renewcommand{\arraystretch}{1.12}
	\begin{tabularx}{\columnwidth}
		{@{}l@{\hspace{0.75em}}>{\raggedright\arraybackslash}X@{}}
		\toprule
		\textbf{Symbol} & \textbf{Description} \\
		\midrule
		$\mathcal{O}_i,\,\Sigma_i$
		& Odometry frame and body frame of robot $i$. \\
		$\bm{z}_i^{\mathcal{O}_i},\,\theta_{\Sigma_i}^{\mathcal{O}_i}$
		& Displacement and yaw of robot $i$ relative to its odometry frame. \\
		$\hat{\bm{q}}_{ij}^{\mathcal{O}_i},\,\hat{\bm{Q}}_{\mathcal{O}_i}^{\mathcal{O}_j}$
		& RL: estimate the initial relative pose ($\bm{p}_{ij}^{\mathcal{O}_{i}}(t_{0})$, $\bm{R}_{\mathcal{O}_{i}}^{\mathcal{O}_{j}}$) between robot $i$ and neighboring robot $j$. \\
		$\hat{\bm{q}}_{ij}^{\Sigma_i},\,\hat{\bm{Q}}_{\Sigma_i}^{\Sigma_j}$
		& RL: estimate the real-time relative pose ($\bm{p}_{ij}^{\Sigma_{i}}(t)$, $\bm{R}_{\Sigma_{i}}^{\Sigma_{j}}(t)$) between robot $i$ and neighboring robot $j$. \\
		$\hat{\bm{p}}_i^{\mathcal{O}_0},\,\hat{\bm{R}}_{\mathcal{O}_i}^{\mathcal{O}_0}$
		& CL: coupled distributed estimator to estimate the initial pose ($\bm{p}_{i0}^{\mathcal{O}_{0}}(t_{0})$, $\bm{R}_{\mathcal{O}_{i}}^{\mathcal{O}_{0}}$) of $i$ relative to the anchor $0$. \\
		$\hat{\bm{z}}_i,\,\hat{\bm{R}}_{i}$
		& CL: distributed observer to estimate the anchor displacement $\bm{z}_0^{\mathcal{O}_0}(t)$ and rotation $\bm{R}_{\Sigma_0}^{\mathcal{O}_0}(t)$. \\
		$\hat{\bm{p}}_{i0}^{\Sigma_i},\,\hat{\bm{R}}_{\Sigma_i}^{\Sigma_0}$
		& CL: estimate the real-time pose ($\bm{p}_{i0}^{\Sigma_{i}}(t)$, $\bm{R}_{\Sigma_{i}}^{\Sigma_{0}}(t)$) of $i$ relative to the anchor $0$.\\
		\bottomrule
	\end{tabularx}
\end{table}

\begin{algorithm}[!t]
	\small
	\caption{Distributed RL and CL at robot $i$}
	\label{RLCL}
	\KwIn{Local and neighbor odometry
		$(\bm{z}_{\ell}^{\mathcal{O}_{\ell}},
		\bm{R}_{\Sigma_{\ell}}^{\mathcal{O}_{\ell}})$ for $\ell\in\{i,j\}$,
		relative measurements $(m_{ij},m_{ji})$, and neighbor estimator states.}
	\KwOut{RL estimates
		$(\hat{\bm{q}}_{ij}^{\Sigma_i},\hat{\bm{Q}}_{\Sigma_i}^{\Sigma_j})$
		for each neighbor $j$, and CL estimates
		$(\hat{\bm{p}}_{i0}^{\Sigma_i},\hat{\bm{R}}_{\Sigma_i}^{\Sigma_0})$.}
	\BlankLine
	\While{localization is active}{
		\blue{// Relative localization for neighboring robot $j$\;}
		\For{each neighbor $j$ of $i$ in $\mathcal{G}_c$}{
			\If{a new sample arrives at $t_{k+1}$}{
				Form $(y_{ij}(t_k),\bm{\Phi}_{ij}(t_k))$ in \eqref{rl_model}
				from the available relative measurements and odometry\;
				Update $\hat{\bm{\Theta}}_{ij}(t_{k+1})$ using the current
				and recorded pairs via \eqref{relative estimator}\;
			}
			Compute real-time RL pose $(\hat{\bm{q}}_{ij}^{\Sigma_i},\hat{\bm{Q}}_{\Sigma_i}^{\Sigma_j})$ using \eqref{chat5_real_time_relative_estimator}\;
		}
		\blue{// Cooperative localization for anchor $0$\;}
		Update the coupled estimator 
		$(\hat{\bm{p}}_i^{\mathcal{O}_0},
		\hat{C}_{\mathcal{O}_i}^{\mathcal{O}_0},
		\hat{S}_{\mathcal{O}_i}^{\mathcal{O}_0})$ via \eqref{initial_pose_estimator}\;
		Update anchor odometry observer $(\hat{\bm{z}}_i,\hat{c}_i,\hat{s}_i)$ via \eqref{distributed observer}\;
		Compute real-time CL pose $(\hat{\bm{p}}^{\Sigma_i}_{i0}, \hat{\bm{R}}^{\Sigma_0}_{\Sigma_i})$ using \eqref{cooperative estimator}\;
	}
\end{algorithm}

In practice, the cosine and sine estimates used to construct
$\hat{\bm{R}}_{\mathcal{O}_i}^{\mathcal{O}_0}$ can be projected
onto $[-1,1]$. With this projection, the corresponding orientation
estimation error $\tilde{\bm{R}}_{\mathcal{O}_i}^{\mathcal{O}_0}$
is bounded.
Consider the initial position error dynamics
\eqref{initial_position_error}. The pairwise RL analysis above
shows that the initial relative position errors
$\tilde{\bm{q}}_{ij}^{\mathcal{O}_i}$ are bounded.
Since the true initial relative positions are fixed, their estimates
$\hat{\bm{q}}_{ij}^{\mathcal{O}_i}$ are also bounded.
For each informed robot $i$, the same RL analysis applies to
the pair $(i,0)$ and bounds its relative position and orientation
errors with respect to the anchor.
Thus, every term in the disturbance $\bm{\Xi}_{pi}$ in
\eqref{initial_position_error} is bounded.
Under Assumption~\ref{top}, $(\mathcal{L}+\mathcal{B})$ is positive
definite, so the stacked initial position error dynamics are
exponentially stable in the absence of disturbances.
A bounded disturbance $\bm{\Xi}_p$ therefore yields a bounded
initial position estimation error $\tilde{\bm{p}}^{\mathcal{O}_0}$.

Since the anchor odometry observer \eqref{distributed observer}
does not use relative measurements, its convergence result still
applies, and the bounded initial orientation estimation error
implies a bounded real-time orientation error through
\eqref{cooperative estimator}.
For the real-time position error, the first term on the right-hand
side of \eqref{real_time_position_error} remains bounded. In the
second term, a persistent error in the estimated initial orientation
can produce a position error that grows with the anchor's horizontal
displacement. 
Thus, under the preceding noise assumptions, the real-time
position error remains bounded when the robots operate in a bounded
workspace. 

\footnotesize
\bibliographystyle{IEEEtran}
\bibliography{IEEEabrv,references}

\begin{thebibliography}{10}
\providecommand{\url}[1]{#1}
\csname url@samestyle\endcsname
\providecommand{\newblock}{\relax}
\providecommand{\bibinfo}[2]{#2}
\providecommand{\BIBentrySTDinterwordspacing}{\spaceskip=0pt\relax}
\providecommand{\BIBentryALTinterwordstretchfactor}{4}
\providecommand{\BIBentryALTinterwordspacing}{\spaceskip=\fontdimen2\font plus
\BIBentryALTinterwordstretchfactor\fontdimen3\font minus
  \fontdimen4\font\relax}
\providecommand{\BIBforeignlanguage}[2]{{%
\expandafter\ifx\csname l@#1\endcsname\relax
\typeout{** WARNING: IEEEtran.bst: No hyphenation pattern has been}%
\typeout{** loaded for the language `#1'. Using the pattern for}%
\typeout{** the default language instead.}%
\else
\language=\csname l@#1\endcsname
\fi
#2}}
\providecommand{\BIBdecl}{\relax}
\BIBdecl

\bibitem{ze2025tase}
J.~L{\"u}, K.~Ze, S.~Yue, K.~Liu, W.~Wang, and G.~Sun, ``Concurrent-learning
  based relative localization in shape formation of robot swarms,'' \emph{IEEE
  Trans. Autom. Sci. Eng.}, 2025.

\bibitem{fang2025Auto}
X.~Fang, L.~Xie, and D.~V. Dimarogonas, ``Simultaneous distributed localization
  and formation tracking control via matrix-weighted position constraints,''
  \emph{Automatica}, vol. 175, p. 112188, 2025.

\bibitem{xie2023tac}
X.~Fang, L.~Xie, and X.~Li, ``Integrated relative-measurement-based network
  localization and formation maneuver control,'' \emph{IEEE Trans. Autom.
  Control}, vol.~69, no.~3, pp. 1906--1913, 2023.

\bibitem{lv2024tii}
Y.~Lv, Z.~Wu, H.~Zhang, Z.~Wang, and H.~Yan, ``Local-bearing-based
  prescribed-time distributed localization of multiagent systems with noisy
  measurement,'' \emph{IEEE Trans. Ind. Inform.}, 2024.

\bibitem{van2020Auto}
Q.~Van~Tran, B.~D. Anderson, and H.-S. Ahn, ``Pose localization of
  leader--follower networks with direction measurements,'' \emph{Automatica},
  vol. 120, p. 109125, 2020.

\bibitem{boughellaba2023lcs}
M.~Boughellaba and A.~Tayebi, ``Bearing-based distributed pose estimation for
  multi-agent networks,'' \emph{IEEE Control Syst. Lett.}, vol.~7, pp.
  2617--2622, 2023.

\bibitem{zheng2025Auto}
C.~Zheng, Y.~Mi, H.~Guo, H.~Chen, Z.~Lin, and S.~Zhao, ``Optimal
  spatial--temporal triangulation for bearing-only cooperative motion
  estimation,'' \emph{Automatica}, vol. 175, p. 112216, 2025.

\bibitem{xie2025tro}
L.~Chen, C.~Liang, S.~Yuan, M.~Cao, and L.~Xie, ``Relative localizability and
  localization for multi-robot systems,'' \emph{IEEE Trans. Robot.}, 2025.

\bibitem{zhao2016Auto}
S.~Zhao and D.~Zelazo, ``Localizability and distributed protocols for
  bearing-based network localization in arbitrary dimensions,''
  \emph{Automatica}, vol.~69, pp. 334--341, 2016.

\bibitem{lcs2024bounded}
N.~De~Carli, E.~Restrepo, and P.~R. Giordano, ``Adaptive observer from
  body-frame relative position measurements for cooperative localization,''
  \emph{IEEE Control Syst. Lett.}, 2024.

\bibitem{tnse2025}
S.~Wang, H.~Wang, W.~Che, and Q.~Li, ``Distributed leader-following formation
  control of networked mobile robots via global orientation estimation,''
  \emph{IEEE Trans. Netw. Sci. Eng.}, 2025.

\bibitem{Biggs1993book}
N.~Biggs, \emph{Algebraic graph theory}.\hskip 1em plus 0.5em minus 0.4em\relax
  Cambridge, U.K.: Cambridge university press, 1993.

\bibitem{ze2024relative}
K.~Ze, W.~Wang, S.~Yue, G.~Sun, K.~Liu, and J.~L{\"u}, ``Relative pose
  estimation for nonholonomic robot formation with uwb-io measurements,''
  \emph{arXiv preprint arXiv:2411.05481}, 2024.

\bibitem{ze2023tie}
K.~Ze, W.~Wang, K.~Liu, and J.~L{\"u}, ``Time-varying formation planning and
  distributed control for multiple uavs in clutter environment,'' \emph{IEEE
  Trans. Ind. Electron.}, vol.~71, no.~9, pp. 11\,305--11\,315, 2023.

\bibitem{jacinto2024pegasus}
M.~Jacinto, J.~Pinto, J.~Patrikar, J.~Keller, R.~Cunha, S.~Scherer, and
  A.~Pascoal, ``Pegasus simulator: An isaac sim framework for multiple aerial
  vehicles simulation,'' in \emph{2024 International Conference on Unmanned
  Aircraft Systems (ICUAS)}.\hskip 1em plus 0.5em minus 0.4em\relax IEEE, 2024,
  pp. 917--922.

\end{thebibliography}

\end{document}